\documentclass[sigconf]{acmart}
\usepackage{enumitem}
\usepackage{multirow}
\usepackage{amsmath}
\newtheorem{theorem}{Theorem}
\newtheorem{lemma}{Lemma}
\newtheorem{definition}{Definition} 
\newtheorem{problem}{Problem}
\usepackage{bm}
\usepackage[mathscr]{euscript}
\usepackage[linesnumbered,ruled,vlined]{algorithm2e}
\usepackage{subfigure}
\usepackage{balance}
\usepackage{float}
\AtBeginDocument{%
  \providecommand\BibTeX{{%
    \normalfont B\kern-0.5em{\scshape i\kern-0.25em b}\kern-0.8em\TeX}}}

\copyrightyear{2020}
\acmYear{2020}
\setcopyright{iw3c2w3}
\acmConference[WWW '20]{Proceedings of The Web Conference 2020}{April 20--24, 2020}{Taipei, Taiwan}
\acmBooktitle{Proceedings of The Web Conference 2020 (WWW '20), April 20--24, 2020, Taipei, Taiwan}
\acmPrice{}
\acmDOI{10.1145/3366423.3380188}
\acmISBN{978-1-4503-7023-3/20/04}
\begin{document}

\title{Generalizing Tensor Decomposition for N-ary Relational Knowledge Bases}

\author{Yu Liu}
\affiliation{%
  \institution{BNRist, EE, Tsinghua University}
	\state{Beijing}
	\country{China}
}
\email{liuyu2419@126.com}

\author{Quanming Yao}
\affiliation{%
  \institution{4Paradigm Inc}
  \country{Hong Kong}}
\email{yaoquanming@4paradigm.com}

\author{Yong Li}
\affiliation{%
  \institution{BNRist, EE, Tsinghua University}
	\state{Beijing}
	\country{China}
}
\email{liyong07@tsinghua.edu.cn}
%
%
%
%
%

\renewcommand{\shortauthors}{Trovato and Tobin, et al.}
\fancyhead{}
\begin{abstract}
		With the rapid development of knowledge bases (KBs), link prediction task, which completes KBs with missing facts, has been broadly studied in especially binary relational KBs (a.k.a knowledge graph) with powerful tensor decomposition related methods. However, the ubiquitous n-ary relational KBs with higher-arity relational facts are paid less attention, in which existing translation based and neural network based approaches have weak expressiveness and high complexity in modeling various relations. Tensor decomposition has not been considered for n-ary relational KBs, while directly extending tensor decomposition related methods of binary relational KBs to the n-ary case does not yield satisfactory results due to exponential model complexity and their strong assumptions on binary relations. To generalize tensor decomposition for n-ary relational KBs, in this work, we propose GETD, a generalized model based on Tucker decomposition and Tensor Ring decomposition. The existing negative sampling technique is also generalized to the n-ary case for GETD. In addition, we theoretically prove that GETD is fully expressive to completely represent any KBs.  Extensive evaluations on two representative n-ary relational KB datasets demonstrate the superior performance of GETD, significantly improving the state-of-the-art methods by over 15\%. Moreover, GETD further obtains the state-of-the-art results on the benchmark binary relational KB datasets.
\end{abstract}




\maketitle

\section{Introduction} \label{sec:intro}
In the past decade, the emerging of numerous web-scale knowledge bases (KBs) such as Freebase \cite{freebase}, Wikidata \cite{wikidata}, YAGO \cite{yago} and Google's Knowledge Graph (KG) \cite{google-KG}, has inspired various applications, e.g., question answering \cite{lukovnikov2017neural}, recommender systems \cite{zhang2016collaborative} and natural language processing (NLP) \cite{logan2019barack}. Most of these KBs are constructed based on binary relations with triplet facts represented as (\emph{head entity, relation, tail entity}). However, due to enormous missing facts, KBs face a fundamental issue of incompleteness, which drives KB completion researches especially \emph{link prediction}: predicting whether two entities are related based on known facts in KBs \cite{balavzevic2019tucker}. Extensive studies have been proposed to address this problem, including translational distance models \cite{transE,transR,transH,tranSparse}, neural network models \cite{socher2013reasoning,convE,schlichtkrull2018modeling}, and tensor decomposition models \cite{DistMult,ComplEx,SimplE,balavzevic2019tucker,RESCAL}. Among them, the mathematically principled tensor decomposition models achieve the best performance, with strong capability to capture latent interactions between entities and relations in KBs.

Despite the great attention in the binary relational KBs, higher-arity relational KBs are less studied. In fact, n-ary a.k.a. multi-fold relations play an important role in KBs. For instance, \emph{Purchase} is a common ternary (3-ary) relation, involved with a \emph{Person}, a \emph{Product}, and a \emph{Seller}. \emph{Sports\_award} is a 4-ary relation, involved with a \emph{Player}, a \emph{Team}, an \emph{Award} and a \emph{Season}, giving an example of \emph{Michael Jordan from Chicago Bulls was awarded the MVP award in 1991-1992 NBA season.} Also, as observed in \cite{wen2016representation}, more than 1/3 of the entities in Freebase participate in the n-ary relation. Besides, since higher-arity relations with more knowledge are closer to natural language, link prediction in n-ary relational KBs provides an excellent potential for question answering related NLP applications \cite{ernst2018highlife}. 

To handle link prediction in n-ary relational KBs directly, two categories of models based on translational distance and neural network are proposed recently. In terms of translational distance models, m-TransH \cite{wen2016representation} directly extends TransH \cite{transH} for binary relations to the n-ary case. RAE \cite{zhang2018scalable} further integrates m-TransH with multi-layer perceptron (MLP) by considering the relatedness of entities. However, since the distance-based scoring function of TransH enforces constraints on relations, it fails to represent some binary relations in KBs \cite{SimplE}. Accordingly, its extensions of m-TransH and RAE are not able to represent some n-ary relations, which impairs the performance. In terms of neural network models, NaLP \cite{guan2019link} leverages neural networks for n-ary relational fact modeling and entity relatedness evaluation, and obtains the state-of-the-art results. Nevertheless, NaLP owes good performance to an enormous amount of parameters, which contradicts the linear time and space requirement for relational models in KBs \cite{bordes2013irreflexive}. Therefore, existing methods do not provide an efficient solution for link prediction in n-ary relational KBs, and it is still an open problem to be addressed.

Although tensor decomposition models have been proved to be very powerful in binary relational KBs by both the state-of-art results \cite{balavzevic2019tucker} and theoretical guarantees on full expressiveness \cite{SimplE,wang2018multi}, there is no work to our knowledge adopting this type of model for link prediction in n-ary relational KBs. A possible way is to extend current tensor decomposition models from the binary case to the n-ary case, while direct extensions yield serious issues. First, several existing models \cite{SimplE,ComplEx} leverage some tricky operations in scoring functions for great performance, while these operations are constrained on binary relations, which are not able to be applied in n-ary relations. Second, powerful tensor decomposition models \cite{balavzevic2019tucker} introduce exponential model complexity with the increase of arity, which cannot be applied in large-scale KBs.

To solve link prediction problem in n-ary relational KBs as well as tackle above challenges, we generalize tensor decomposition for n-ary relational KBs. Specifically, we first extend TuckER \cite{balavzevic2019tucker}, the state-of-the-art model for binary relations, to n-TuckER for n-ary relations, with Tucker decomposition utilized \cite{tucker1966some}. Note that the higher-order core tensor in n-TuckER grows exponentially with the arity, and excessively complex models usually overfit, which implies poor performance. Thus, motivated by the benefits of tensor ring (TR) decomposition \cite{zhao2016tensor} for neural network compression in computer vision \cite{widecompression,pan2019compressing}, we integrate TR with n-TuckER together. By representing the higher-order core tensor into a sequence of 3rd-order tensors, TR significantly reduces the model complexity with performance enhanced. The overall model is termed as GETD, \underline{Ge}neralized \underline{T}ensor \underline{D}ecomposition for n-ary relational KBs. Since most KBs only provide positive observations, we also generalize the existing negative sampling technique for efficiently training on GETD. Considering the importance for a link prediction model to have enough expressive power to represent various true and false facts in KBs, we further prove that GETD is \emph{fully expressive}.

The main contributions of this paper are summarized as follows:
\begin{itemize}[leftmargin=10px]
	\item We investigate tensor decomposition for link prediction in n-ary relational KBs, and identify the bottleneck of directly extending existing binary relational models to the n-ary case, including the binary relation constrained scoring function and exponential model complexity.
	
	\item We propose GETD, a generalized tensor decomposition model for n-ary relational KBs. GETD integrates TR decomposition with Tucker decomposition, which scales well with both the size and the arity of the KB. We also generalize the negative sampling technique from the binary to the n-ary case. To the best of our knowledge, GETD is the first model that leverages tensor decomposition techniques for n-ary relational link prediction.
	
	\item We prove that GETD is fully expressive for n-ary relational KBs, which is able to cover all types of relations, and can completely separate true facts from false ones. 
	
	\item We conduct extensive experiments on two representative n-ary relational KB datasets. The results demonstrate that GETD outperforms state-of-the-art n-ary relational link prediction models by 15\%. Furthermore, GETD achieves close and even better performance on two standard binary relational KB datasets compared with existing state-of-the-art models.
\end{itemize}

We organize the rest of this paper as follows. Section \ref{sec:related work} gives a systematic review on the related works of link prediction in KBs. Section \ref{sec:background} introduces the background of tensor decomposition and notations. After that, the framework of GETD and theoretical analyses are presented in Section \ref{sec:framwork}. Section \ref{sec:experiments} evaluates the performance on representative KB datasets and provides extensive analyses. In light of our results, this paper is concluded in Section \ref{sec:conclusion}.

\section{Related Work}\label{sec:related work}
We classify the related works into two categories of binary relational KBs and n-ary relational KBs.

\subsection{Link Prediction on Binary Relational KBs}
Basically, link prediction models embed entities and relations into low-dimensional vector spaces, and define a scoring function with embeddings to measure if a given fact is true or false. Based on the scoring function design, the typical works in binary relational KBs can be categorized into three groups: translational distance models \cite{transE,transR,transH,tranSparse}, neural network models \cite{socher2013reasoning,convE,schlichtkrull2018modeling}, and tensor decomposition models \cite{DistMult,ComplEx,SimplE,balavzevic2019tucker,RESCAL}.
	
	Translational distance models measure the entity distance after a translational operation carried out by the relation \cite{wang2017knowledge}, and various translational operations are exploited for distance-based scoring functions \cite{transE,transR,transH,tranSparse}. However, most translational distance models are found to have restrictions on relations \cite{SimplE,wang2017knowledge}, thus can only represent part of relations.
	
	Neural network models \cite{socher2013reasoning,convE,schlichtkrull2018modeling} subtly design the scoring function with various neural network structures, which always require a great many parameters to completely represent all relations \cite{SimplE,wang2017knowledge}, increasing training complexity and impractical for large-scale KBs.

With solid theory and great performance, tensor decomposition models are more prevalent methods. In this aspect, the link prediction task is framed as a 3rd-order binary tensor completion problem, where each element corresponds to a triple, one for true facts while zero for false/missing facts respectively. Thus, various tensor decomposition models are proposed to approximate the 3rd-order tensor. For example, DistMult \cite{DistMult} uses Canonical Polyadic (CP) decomposition \cite{hitchcock1927expression} with the equivalence of head and tail embeddings for the same entity, however, fails to capture the asymmetric relation. Furthermore, SimplE \cite{SimplE} takes advantage of the inverse of relations to address the asymmetric relation, while ComplEx \cite{ComplEx} leverages complex-valued embeddings for solution. Recently, Tucker decomposition \cite{tucker1966some} is adopted in TuckER \cite{balavzevic2019tucker} for link prediction, and achieves the state-of-the-art performance. Compared with former works only using entity and relation embeddings to capture the knowledge in KBs, TuckER additionally introduces the core tensor to model interactions between entities and relations, which further improves the expressiveness. According to the discussion, generalizing tensor decomposition is promising for n-ary relational link prediction.

\subsection{Link Prediction on N-ary Relational KBs}
Existing works on n-ary relational link prediction can be categorized into two classes based on the scoring function: translational distance models \cite{wen2016representation,zhang2016collaborative} and neural network models \cite{guan2019link}.
	
The translational distance models of m-TransH \cite{wen2016representation} and RAE \cite{zhang2018scalable} are the first series of works in this field. Based on the distance translation idea, m-TransH is proposed by extending TransH \cite{transH} for the n-ary case, where entities are all projected onto the relation-specific hyperplane, and the scoring function is defined by the weighted sum of projection results. RAE further improves m-TransH with the relatedness assumption that, the likelihood of two entities co-participating in a common n-ary relational fact is important for link prediction. MLP is utilized to model the relatedness and coupled into the scoring function. Since these models are directly extended from the binary case, the restrictions on relations are also inherited with limited representation capability to KBs.  
	
The neural network model, NaLP \cite{guan2019link}, is recently proposed for the state-of-the-art performance in n-ary relational link prediction. In NaLP, entity embeddings of an n-ary relational fact are first passed to the convolution for feature extraction, then the overall relatedness is modeled by FCN, whose output represents the evaluation score. However, a great many parameters involved in NaLP make the training intractable. Moreover, the latent connections between similar relations are not considered, which further leads to inferior empirical performance.
	
As previously discussed, tensor decomposition is a potential solution for n-ary relational link prediction, while directly extending current binary relational tensor decomposition models to the n-ary case is challenging with various bottlenecks. First, most CP-based models achieve great performance mainly due to carefully designed scoring functions with tricky operations. For instance, to model all types of relations, the relation inverse in SimplE \cite{SimplE}  and complex-valued embeddings in ComplEx \cite{ComplEx} are all binary relation constrained operations, which cannot find equivalents when it comes to the n-ary case. Second, some direct extensions introduce tremendous parameters like TuckER \cite{balavzevic2019tucker} to the n-ary case with exponential model complexity, which is impractical and easily affected by noise \cite{SimplE,ComplEx}. Besides, other models like DistMult \cite{DistMult} force the relation to be symmetric, thus are not able to completely represent n-ary relational KBs. A recent work \cite{fatemi2019knowledge} explores DistMult with convolution to the n-ary case, but the interaction of entities and relations is not fully captured. Through our investigation, GETD is the first generalized tensor decomposition model for n-ary relational KBs with both performance and model complexity satisfied.

\section{Background and Notation}\label{sec:background}
\subsection{Tensors and Notations}

\begin{table}[t]
	\caption{Notations}
	\vspace{-10px}
	\label{tab:notations}
	\begin{tabular}{c|l}
		\toprule
		Symbol & Definition\\
		\midrule
		$\bm{\mathscr{X}}$ & $n$th-order tensor $\in\mathbb{R}^{I_1\times\cdots\times I_n}$  \\
		$x_{i_1i_2\cdots i_n}$ & ($i_1,i_2,\cdots,i_n$)-th element of $\bm{\mathscr{X}}$\\
		$\bm{\mathscr{G}}$ & $n$th-order core tensor $\in\mathbb{R}^{J_1\times\cdots\times J_n}$ \\
		$\bm{A}^{(k)}$ & k-mode factor matrix $\in\mathbb{R}^{I_n\times J_n}$\\
		$\bm{a}^{(k)}_j$  & $j$-th column vector of $\bm{A}^{(k)}$ \\
		$\bm{\mathscr{Z}}_k$ & $k$-th TR latent tensor $\in\mathbb{R}^{r_k\times n_k\times r_{k+1}}$\\
		$\bm{Z}_k(i_k)$ &  $i_k$-th lateral slice matrix of $\bm{\mathscr{Z}}_k,\in\mathbb{R}^{r_k\times r_{k+1}}$  \\
		$\bm{r}=[r_1,r_2,\cdots,r_n]$  & TR-ranks \\
		$n_e, \, n_r$  	& the number of entities/relations in the KB\\
		$d_e, \, d_r$  			& entity/relation embedding dimensionality \\
		$n_i$      & 2nd-mode dimensionality of $\bm{\mathscr{Z}}_i$\\		
		$\circ$ &    vector outer product\\
		$\times_n$ &  tensor $n$-mode product\\
		$\langle \cdot \rangle$ & multi-linear dot product\\
		$trace\{\cdot\}$ &  matrix trace operator\\
		\bottomrule
	\end{tabular}
	\Description[]{The table of notations used in this paper.} 
\end{table}

A tensor is a multi-order array, which generalizes the scalar (0th-order tensor), the vector (1st-order tensor) and the matrix (2nd-order tensor) to higher orders. We represent scalars with lowercase letters, vectors with boldface lowercase letters, matrices with boldface uppercase letters and higher-order tensors with boldface Euler script letters. For indexing, let $\bm{a}_i$ denote the $i$-th column of a matrix $\bm{A}$, $x_{i_1i_2\cdots i_n}$ denote the ($i_1,i_2,\cdots,i_n$)-th element of a higher-order tensor $\bm{\mathscr{X}}\in\mathbb{R}^{I_1\times\cdots\times I_n}$, where $I_i$ is the dimensionality of the $i$-th mode. Especially, given a 3rd-order tensor $\bm{\mathscr{Z}}\in\mathbb{R}^{I_1\times I_2\times I_3}$, the $i_2$-th lateral slice matrix of $\bm{\mathscr{Z}}$ is denoted by $\bm{Z}(i_2)$ in the size of $I_1\times I_3$, a.k.a., $\bm{Z}_{:i_2:}$ where the colon indicates all elements of a mode.

As for the operation on tensors, $\circ$ represents the vector outer product, and $\times_i$ represents the tensor $i$-mode product. $\langle\cdot\rangle$ represents the multi-linear dot product, written as $\langle \bm{a}^{(1)},\bm{a}^{(2)},\cdots,\bm{a}^{(n)}\rangle=\sum_i \bm{a}^{(1)}_i\bm{a}^{(2)}_i\cdots\bm{a}^{(n)}_i$. $trace\{\cdot\}$ is the matrix trace operator, written as $trace\{\bm{A}\}=\sum_i a_{ii}$. More details about these operations and tensor properties can be referred to \cite{kolda2009tensor}. The related notations frequently used in this paper are listed in Table~\ref{tab:notations}.

\begin{figure*}[t]
	\centering
	\includegraphics[width=.9\linewidth]{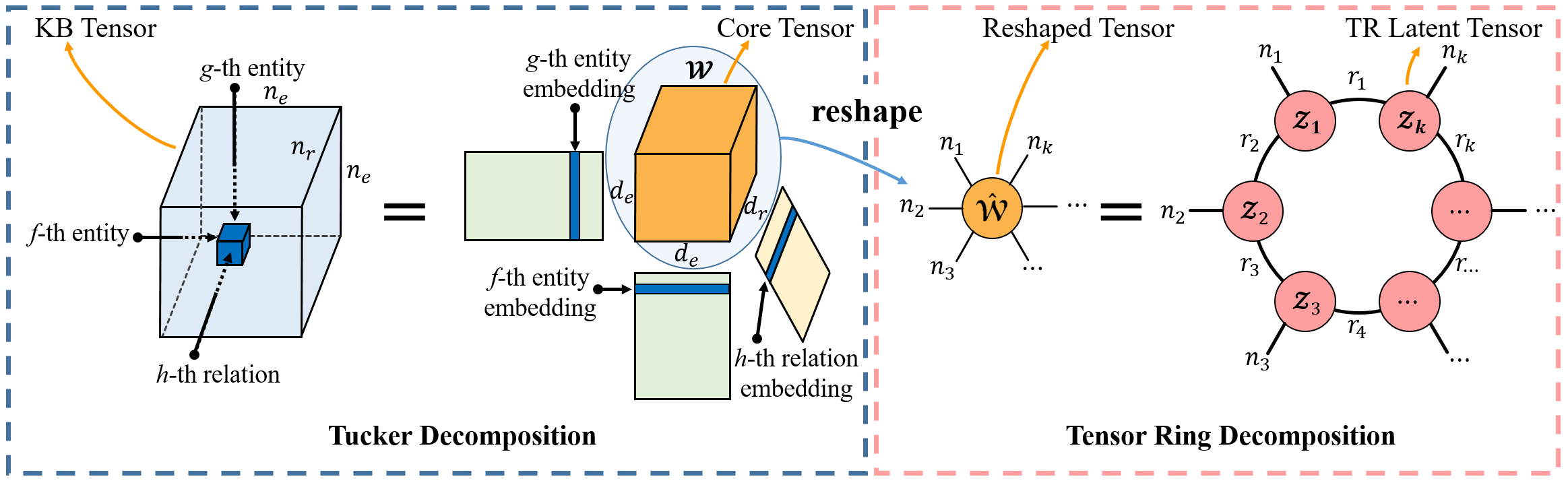}
	\vspace{-10px}
	\caption{The main framework of GETD model.}
	\label{fig:framework}
	\Description[]{The main framework of GETD model, including Tucker decomposition and Tensor ring decomposition.} 
\end{figure*}

\subsection{Tucker Decomposition}
Tucker decomposition was initially proposed for three-order tensor decomposition \cite{tucker1966some}. It can be generalized to higher order, which decomposes a higher-order tensor into a set of factor matrices and a relatively small core tensor. Given an $n$th-order tensor $\bm{\mathcal{X}}\in \mathbb{R}^{I_1\times\cdots\times I_n}$, Tucker decomposition can be denoted as,
\begin{align}
\bm{\mathcal{X}}
&\approx \bm{\mathcal{G}}\times_1 \bm{A}^{(1)}\times_2 \bm{A}^{(2)}\cdots \times_n \bm{A}^{(n)} \notag\\
&=\sum^{J_1}_{j_1=1}\sum^{J_2}_{j_2=1}\cdots\sum^{J_n}_{j_n=1} g_{j_1 j_2 \cdots j_n} \bm{a}^{(1)}_{j_1}\circ \bm{a}^{(2)}_{j_2}\circ\cdots\circ \bm{a}^{(n)}_{j_n},
\end{align}
where $\bm{\mathcal{G}}\in\mathbb{R}^{J_1\times \cdots\times J_n}$ is the \emph{core tensor}, $J_k$ is the rank of $k$-th mode, and $\{\bm{A}^{(k)}\,|\,\bm{A}^{(k)}\in\mathbb{R}^{I_k\times J_k}\}^n_{k=1}$ is the set of \emph{factor matrices}. Usually, $J_1,\cdots,J_n$ are smaller than $I_1,\cdots,I_n$. Thus the number of parameters is reduced compared with the approximated tensor $\bm{\mathscr{X}}$.

\subsection{Tensor Ring (TR) Decomposition}
Although Tucker decomposition approximates a higher-order tensor with fewer parameters, the number of parameters scales exponentially to the tensor order. Tensor ring (TR) decomposition \cite{zhao2016tensor}, on the other hand, represents a higher-order tensor by a sequence of 3rd-order latent tensors multiplied circularly. Given an $n$th-order tensor $\bm{\mathcal{X}}\in \mathbb{R}^{I_1\times\cdots\times I_n}$, TR decomposition can be expressed in an element-wise form as,
\begin{align}
\!\!\!	x_{i_1i_2\cdots i_n} 
\!	&\approx\! trace\{\bm{Z}_1(i_1)\bm{Z}_2(i_2)\cdots \bm{Z}_n(i_n)\}\! =\! trace\{\prod^n_{k=1}\bm{Z}_k(i_k)\},
\end{align}
where $\{\bm{\mathscr{Z}}_k\,|\,\bm{\mathscr{Z}}_k\in \mathbb{R}^{r_k\times I_k\times r_{k+1}}, r_1=r_{n+1}\}^n_{k=1}$ is the set of TR latent tensors, and $\bm{Z}_k(i_k)$ is in the size of $\mathbb{R}^{r_k\times r_{k+1}}$ accordingly. For convenience, we also denote the above TR decomposition as $\bm{TR}(\bm{\mathscr{Z}}_1,\cdots,\bm{\mathscr{Z}}_k)$. Especially, the size of latent factors, concatenated and denoted by $\bm{r}=[r_1,r_2,\cdots,r_n]$ is called \emph{TR-ranks}.

\section{GETD: Design and Model}\label{sec:framwork}

Borrowing the concept of n-ary relation \cite{codd1970relational,ernst2018highlife}, the n-ary relational fact can be defined as follows,
\begin{definition}[n-ary relational fact]
	Given an n-ary relational KB with the set of relations $\mathcal{R}$ and the set of entities $\mathcal{E}$, an n-ary relational fact is an $(n+1)$-tuple $({i_r},{i_1},{i_2},\cdots,{i_n})\subseteq \mathcal{R}\times \mathcal{E}\times \mathcal{E}\times\cdots\times \mathcal{E}$ where $\mathcal{R}$ and $\mathcal{E}$ are called relation domain and entity domain.
\end{definition}

Especially, ${i_k}$ is the $k$-th entity to the relation ${i_r}$, belonging to the $k$-th entity domain. In the binary case of $({i_r},{i_1},{i_2})$, ${i_1}$ and ${i_2}$ are head entity and tail entity, and ${i_r}$ is the relation, respectively. 

Then, the link prediction problem can be defined as follows,
\begin{problem}[link prediction]
	Given an incomplete n-ary relational KB $\mathcal{S}=\{({i_r},{i_1},{i_2}\cdots,{i_n})\}$, the link prediction problem aims to infer missing facts based on $\mathcal{S}$. 
\end{problem}
In practice, given the relation and any $n-1$ entities in an n-ary relational fact, the problem is simplified as predicting the missing entity, e.g. predicting the 1-st entity of the incomplete n-ary relational fact $({i_r},?,{i_2},{i_3},\cdots,{i_n})$.

\subsection{Rethinking Tucker for KBs}
From the point view of tensor completion, an n-ary relational KB can be represented as a binary valued $(n+1)$th-order KB tensor $\bm{\mathscr{X}}\in \{0,1\}^{n_r\times n_e\times n_e\times \cdots \times n_e}$ ($n_r=\vert\mathcal{R}\vert, n_e=\vert\mathcal{E}\vert$), whose 1st-mode is the relation mode, while the other modes are entity modes in the n-ary relational fact. $x_{i_ri_1i_2\cdots i_n}$ equal to one means the specific n-ary relational fact is true, and zero for false/missing. The approximated low-rank scoring tensor is denoted by $\bm{\hat{\mathscr{X}}}\in\mathbb{R}^{n_r\times n_e\times n_e\times \cdots \times n_e}$. Accordingly, the link prediction on $({i_r},?,{i_2},{i_3},\cdots,{i_n})$ can be answered by the entity with maximum value or score in corresponding mode vector of the scoring tensor.

Especially, the state-of-the-art binary relational link prediction model TuckER \cite{balavzevic2019tucker} can be directly extended to the n-ary case termed as n-TuckER, with relation embedding matrix $\bm{R}=\bm{A}^{(1)}\in\mathbb{R}^{n_r\times d_r}$, and entity embedding matrix $\bm{E}$ that is equivalent for each mode entities, i.e., $\bm{E}=\bm{A}^{(2)}=\cdots=\bm{A}^{(n+1)}\in\mathbb{R}^{n_e\times d_e}$, where $d_r$ and $d_e$ represent the dimensionality of relation and entity embedding vectors respectively. The scoring function is defined as,
\begin{align}
\phi({i_r},{i_1},{i_2},\cdots,{i_n})
&=\hat{x}_{i_r i_1 i_2\cdots i_n} \notag\\
&=\bm{\mathscr{W}}\times_1 \bm{r}_{i_r} \times_2 \bm{e}_{i_1}\times_3 \bm{e}_{i_2}\cdots \times_{n+1}\bm{e}_{i_n},  \label{eq:tucker_score_func}
\end{align}
where $\bm{\mathscr{W}}\in\mathbb{R}^{d_r\times d_e\times d_e\times\cdots\times d_e}$ is the $(n+1)$th-order core tensor, $\bm{r}_{i_r}$ and $\{\bm{e}_i\}^{i_n}_{i=i_1}$ are the rows of $\bm{R}$ and $\bm{E}$ representing the relation and the entity embedding vectors. {\color{black}Such a straightforward design inevitably leads to a model complexity of
	\begin{align}
	\mathcal{O}(n_ed_e+n_rd_r+d^n_ed_r), \notag
	\end{align}
	which grows exponentially with $d_e$.} Besides the unacceptable complexity in parameters and increased training difficulty, n-TuckER also faces the dilemma that, excessively complex models are easily affected by noise and prone to overfitting, leading to poor testing performance \cite{SimplE,ComplEx}. 

\subsection{The GETD Model}
In this part, the model construction of GETD with the scoring function is first introduced. Since the existing negative sampling technique only limits on binary relational KBs, then we deal with negative samples for the n-ary case.
\subsubsection{The Scoring Function}
Despite the model complexity and overfitting, leveraging the $(n+1)$th-order core tensor to capture the interaction of entities and relations is instructive that the similarity between entities and relations is encoded in core tensor element. Such Tucker interaction way ensures the strong expressive capability of representing various facts in KBs. It can be envisioned that a model with the Tucker interaction way as well as low complexity is promising for link prediction in n-ary relational KBs.

To achieve this, TR decomposition draws our attention that a higher-order tensor can be decomposed by quite a few parameters in 3rd-order latent tensor sequences. This motivates the general construction of GETD. 

First, in the outer layer of GETD, the original KB tensor is decomposed via Tucker decomposition following \eqref{eq:tucker_score_func}, which reserves Tucker interaction way as well as strong expressiveness. Subsequently, in the inner layer, the intermediate core tensor $\bm{\mathscr{W}}$ is flexibly reshaped to a $k$th-order tensor $\bm{\hat{\mathscr{W}}}\in\mathbb{R}^{n_1\times\cdots\times n_k}$ with $\prod_{i=1}^k n_i=d^n_ed_r$ ($k\geq n+1$) satisfied. Then TR decomposes the reshaped tensor $\bm{\hat{\mathscr{W}}}$ into $k$ latent 3rd-order tensors $\{\bm{\mathscr{Z}}_i\,|\,\bm{\mathscr{Z}}_i\in \mathbb{R}^{r_i\times n_i\times r_{i+1}}, r_1=r_{k+1}\}^k_{i=1}$, reducing the number of parameters. The main framework of GETD is shown in Figure~\ref{fig:framework} (in n=2 case). Specifically, the left part of the figure depicts the construction of outer layer with Tucker decomposition, while the right part presents the TR construction procedure of inner layer. The corresponding expression is,
\begin{align}
\hat{w}_{j_1j_2\cdots j_k}= trace\{\bm{Z}_1(j_1)\bm{Z}_2(j_2)\cdots \bm{Z}_k(j_k)\}.
\end{align}

Overall, the scoring function of GETD can be expressed as,
\begin{align}
&\phi({i_r},{i_1},{i_2},\cdots,{i_n})
=\bm{\hat{\mathscr{W}}}\times_1 \bm{r_{i_r}} \times_2 \bm{e}_{i_1}\times_3 \bm{e}_{i_2}\cdots \times_{n+1}\bm{e}_{i_n}\notag \\
=&\bm{TR}(\bm{\mathscr{Z}}_1,\cdots,\bm{\mathscr{Z}}_k)\times_1 \bm{r_{i_r}} \times_2 \bm{e}_{i_1}\times_3 \bm{e}_{i_2}\cdots \times_{n+1}\bm{e}_{i_n} \label{eq:score}.
\end{align}

The model complexity of GETD is 
\begin{align}
\mathcal{O}(n_ed_e+n_rd_r+kn^3_{\max}),\ \ \ \ \text{s.t.} \ \   n_{\max}=\max\limits_{i=1,\cdots,k}n_i, \notag
\end{align}
which is much lower than n-TuckER, and discussed in detail later. Accordingly, with Tucker interaction way as well as low model complexity, GETD not only guarantees strong expressive capability, but also avoids the overfitting problem with many parameters, which improves the testing performance.

\subsubsection{Dealing with Negative Samples}
In KBs, we usually only have positive observations, i.e., which relation exists among different sets of entities. Thus, even with the designed scoring function in \eqref{eq:score}, we cannot train an embedding model due to lack of negative observations. In embedding of binary relational KBs, given a positive triplet $(i_r, i_1, i_2)$, good candidates of negative samples \cite{nickel2015review} are
	\begin{align}
	\!\!\!\!\!\!\! \mathcal{N}_{(i_r, i_1, i_2)}
	&\equiv 
	\mathcal{N}_{(i_r, i_1, i_2)}^{(1)} \cup \mathcal{N}_{(i_r, i_1, i_2)}^{(2)} \notag \\
	&\equiv
	\left\lbrace (i_r, \bar{i_{1}}, i_2) \not\in \mathcal{S}\, | \, \bar{i_{1}} \in \mathcal{E} \right\rbrace 
	\cup
	\left\lbrace (i_r, i_1, \bar{i_{2}})\not\in \mathcal{S} \,|\, \bar{i_{2}} \in \mathcal{E} \right\rbrace.\label{eq:negset} 
	\end{align} 
Then, negative samples can be sampled from \eqref{eq:negset} by either fixed or dynamic distribution \cite{zhang2019nscaching}. More recently, multi-class log-loss \cite{joulin2017fast,lacroix2018canonical} has developed as a replacement for the above sampling scheme, which can offer better learning performance. Specifically, it considers all candidates in \eqref{eq:negset} simultaneously, i.e.,
	\begin{align}
	\vspace{-10px}
	\mathcal{L}_{(i_r, i_1, i_2)} = \mathcal{L}^{(1)}_{(i_r, i_1, i_2)}+\mathcal{L}^{(2)}_{(i_r, i_1, i_2)}
	\label{eq:mclog}, 
	\end{align}
	where
	\vspace{-10px}
	\begin{align*}
	\mathcal{L}^{(j)}_{(i_r,i_1, \cdots, i_n)}
	= -\phi(i_r,i_1, \cdots, i_n)+
	\log
	\left( 
	e^{\phi(i_r,i_1, \cdots, i_n)} 
	\! + \!\!\!\!\!\!\!\!\! \sum_{x\in\mathcal{N}^{(j)}_{(i_r,i_1, \cdots, i_n)}}  \!\!\!\!\!\!\!\!\! e^{\phi(x)}
	\right).
	\end{align*}
Here, we extend the above multiclass log-loss to n-ary relational KBs. For one positive n-ary relational fact $(i_r, i_1, i_2,\cdots,i_n)$, $n$ groups of negative sample candidates are generated from corresponding $n$ entity domains, defined as
\vspace{-5px}
	\begin{align}
	\!\!\!
	\mathcal{N}_{(i_r, i_1, i_2,\cdots,i_n)}
	\equiv
	\bigcup_{m=1}^n 
	\left\lbrace (i_r, \cdots, \bar{i_{m}}, \cdots) \not\in \mathcal{S}\, | \, \bar{i_{m}} \in \mathcal{E} \right\rbrace. \label{eq:negset-nary}
	\end{align}
Accordingly, with negative samples given in \eqref{eq:negset-nary}, the loss function of GETD is defined as,
	\begin{align}
	\mathcal{L}_{(i_r, i_1, i_2,\cdots,i_n)} 
	= \sum_{j = 1}^n \mathcal{L}^{(j)}_{(i_r, i_1, i_2,\cdots,i_n)}
	\label{eq:mclog-nary}.
	\end{align}

\begin{algorithm}[t]
	\caption{Training Algorithm for GETD.} \label{alg:training}
	\KwIn{{training set
			$\mathcal{S}=\{({i_r},{i_1},{i_2}\cdots,{i_n})\}$,\\
			\hspace{2.78em} reshaped tensor order $k$, TR-ranks $\bm{r}$,\\
			\hspace{2.78em} entity/relation embedding dimension $d_e/d_r$;}}
	
	\raggedright{initialize embeddings $\bm{E}$, $\bm{R}$ for $e\in\mathcal{E}$ and $rel\in\mathcal{R}$, TR latent tensors $\{\bm{\mathscr{Z}}_i\}_{i=1}^k$}; \\
	
	\For{$t=1,2,\cdots,n_{epoch}$}
	{   sample a mini-batch $\mathcal{S}_{\text{batch}}\subseteq \mathcal{S}$ of size $m_b$;\\
		$\mathcal{L}\leftarrow 0$; \\
		\For{$({j_r},{j_1},{j_2},\cdots,{j_{n}})\in\mathcal{S}_{\textnormal{batch}}$}
		{	\raggedright{construct negative sample set $\mathcal{N}_{(i_r, i_1, i_2,\cdots,i_n)}$;}\\
			
			\raggedright{${\phi({j_r},{j_1},{j_2},\cdots,{j_{n}})}\leftarrow$ compute the score using \eqref{eq:score};}\\
			
			\raggedright{$\mathcal{L}({j_r},{j_1},{j_2},\cdots,{j_{n}})\leftarrow$ compute the loss using \eqref{eq:mclog-nary}}
			
			\raggedright{$\mathcal{L}\leftarrow \mathcal{L}+\mathcal{L}({j_r},{j_1},{j_2},\cdots,{j_{n}})$;\\}
		}	
		\raggedright{update parameters of embeddings and TR latent tensors w.r.t. the gradients using $\nabla \mathcal{L}$;}
	}
	\KwOut{embeddings $\bm{E},\bm{R}$ and TR latent tensors $\{\bm{\mathscr{Z}}_i\}_{i=1}^k$.}
	\Description[]{The algorithm description.} 
\end{algorithm}

\subsection{Training}
GETD is trained in a mini-batch way, where all observed facts and each entity domain therein are considered for training. Algorithm \ref{alg:training} presents the pseudo-code of the training algorithm. 
With the embedding dimensions and TR-ranks as input, the embeddings of entities and relations as well as TR latent tensors, are randomly initialized before training in line 1. During the training, line 3 samples a mini-batch $\mathcal{S}_{\text{batch}}$ of size $m_b$, in which each observation is considered for training in lines 4-10.
Specifically, for each n-ary relational fact in $\mathcal{S}_{\text{batch}}$, the algorithm constructs the negative sample set $\mathcal{N}_{(i_r, i_1, i_2,\cdots,i_n)}$ following \eqref{eq:negset-nary}, as shown in line 6. Then, the score of the observation as well as the negative samples are computed using \eqref{eq:score} in line 7, which are further utilized to compute the multiclass log-loss with \eqref{eq:mclog-nary} in lines 8-9. Finally, the algorithm updates the model parameters according to the loss gradients.

\subsection{Complexity Analysis}\label{sec:theoretical_analysis}
\begin{table*}[tbp]
	\caption[ujhj]{Scoring functions of state-of-the-art n-ary relational link prediction models for a given fact $({i_r},{i_1},{i_2},\cdots,{i_n})$, with their expressiveness, and significant terms of their model complexity. $n_e$ and $n_r$ are the number of entities and relations, while $d_e$ and $d_r$ are the dimensionality of entity and relation embeddings respectively. $k$ and $n_{\max}$ are the number and the maximum size of TR latent tensors. $\min(\cdot)$ is the element-wise minimizing operation, $[\cdot,\cdot]$ and $[\cdot;\cdot]$ denote hstack and vstack operation.}
	\label{tab:complexity}
	\vspace{-10px}
	\begin{tabular}{cccc}
		\toprule
		\textbf{Model} & \textbf{Scoring Function} & \textbf{Fully Expressive} & \textbf{Model Complexity} \\
		\midrule
		RAE \cite{zhang2018scalable}& $\|\sum^n_{j=1}a_j(\bm{e}_{i_j}-\bm{w}_{i_r}^\top\bm{e}_{i_j}\bm{w}_{i_r})+\bm{r}_{i_r} \|_p$ & No & $\mathcal{O}(n_ed_e+n_rd_e)$  \\
		
		NaLP \cite{guan2019link}& $FCN_2(\min(FCN_1(Conv([\bm{W}_r,[\bm{e}_{i_1};\bm{e}_{i_2};\cdots;\bm{e}_{i_n}]]))))$ & No & $\mathcal{O}(n_ed_e+nn_rd_r)$ \\
		
		n-CP (extension of \cite{hitchcock1927expression}) & $\langle\bm{r}_{i_r},\bm{e}^{(1)}_{i_1},\bm{e}^{(2)}_{i_2},\cdots,\bm{e}^{(n)}_{i_n}\rangle$ & Yes & $\mathcal{O}(nn_ed_e+n_rd_e)$\\
		
		n-TuckER (extension of \cite{balavzevic2019tucker}) & $\bm{\mathscr{W}}\times_1\bm{r}_{i_r}\times_2\bm{e}_{i_1}\times_3\bm{e}_{i_2}\cdots\times_{n+1}\bm{e}_{i_n}$ & Yes &$\mathcal{O}(n_ed_e+n_rd_r+d^n_ed_r)$\\
		
		GETD (this paper) & $\bm{TR}(\bm{\mathscr{Z}}_1,\cdots,\bm{\mathscr{Z}}_k)\times_1 \bm{r}_{i_r}\times_2\bm{e}_{i_1}\times_3\bm{e}_{i_2}\cdots\times_{n+1}\bm{e}_{i_n}$ & Yes & $\mathcal{O}(n_ed_e+n_rd_r+kn^3_{\max})$ \\
		\bottomrule
	\end{tabular}
\Description[]{The comparison between different models on scoring function, fully expressiveness, and model complexity.} 
\end{table*}

According to the model description, the entity and relation embeddings of GETD cost $\mathcal{O}(n_ed_e+n_rd_r)$ parameters. Since each TR latent tensor is 3rd-order with $\bm{\mathscr{Z}}_i\in \mathbb{R}^{r_i\times n_i\times r_{i+1}}$ and TR-rank $r_i$ is usually smaller than $n_i$, the $k$ TR latent tensors cost $\mathcal{O}(kn^3_{\max})$ parameters in sum. Thus, the model complexity of GETD is obtained as $\mathcal{O}(n_ed_e+n_rd_r+kn^3_{\max})$, while GETD also retains the efficiency benefits of tensor mode product in linear time complexity.
	
Moreover, due to the constraint $\prod_{i=1}^k n_i=d^n_ed_r$ in GETD, if TR latent tensors are in the same shape, we can obtain the equation of $n_{\max}=(d^n_ed_r)^{1/k}$. When applied in large-scale KBs with over thousands of entities \cite{yago,bordes2014semantic}, GETD with high reshaped tensor order (larger $k$) derives that $kn^3_{\max}\ll n_ed_e$, which reduces to the linear model complexity of $\mathcal{O}(n_ed_e+n_rd_r)$ to KB sizes. 

We compare GETD with state-of-the-art n-ary relational link prediction models, in terms of scoring function, expressiveness and the model complexity in Table~\ref{tab:complexity}. Among the models, n-TuckER and n-CP are the extensions of Tucker \cite{balavzevic2019tucker} and CP \cite{hitchcock1927expression}, respectively. n-TuckER costs exponential model complexity due to the higher-order core tensor, which is unacceptable for large-scale n-ary relational KBs. Moreover, the explosion of the number of parameters makes n-TuckER prone to overfitting, as shown in experimental results of Section~\ref{sec:experiments}. n-CP requires different embeddings for one entity in different entity domains, which brings the complexity of $\mathcal{O}(nn_ed_e+n_rd_e)$. As stated before, GETD has linear model complexity $\mathcal{O}(n_ed_e+n_rd_e)$ to KB sizes in practical. Therefore, GETD achieves the lowest model complexity in tensor decomposition models with the best performance.
	
As introduced in Section~\ref{sec:related work}, RAE \cite{zhang2018scalable} is a translational distance model, and NaLP \cite{guan2019link} is a neural network model.	Although these two models achieve similar model complexity to GETD, they are short of expressive power, thus perform badly, which is proved by link prediction performance in Section~\ref{sec:experiments} later. Thus, GETD performs best on both model complexity and expressiveness.

\begin{table}[bp]
	\caption{Dataset Statistics. Here ``-3'' and ``-4'' denote the 3-ary and 4-ary relational KB datasets, respectively.}
	\label{tab:dataset}
	\vspace{-10px}
	\begin{tabular}{cccccc}
		\toprule
		Dataset & \#Entities & $\!\!$\#Relations$\!\!$ & \#Train & \#Valid & \#Test \\
		\midrule
		WikiPeople-3 & 12,270       & 66           & 20,656 & 2,582  & 2,582 \\
		WikiPeople-4 & 9,528        & 50           & 12,150 & 1,519  & 1,519 \\
		JF17K-3      & 11,541       & 104          & 27,635 & 3,454  & 3,455 \\
		JF17K-4      & 6,536        & 23           & 7,607  & 951   & 951 \\
		Synthetic10-3  & 10          & 2           & 400  & 50   & 50 \\
		Synthetic10-4  & 10          & 2           & 1,200  & 150   & 150 \\
		WN18         &  40,943       & 18 			& 141,442 & 5,000 & 5,000\\
		FB15k 		& 14,951	& 1,345			& 483,142 & 50,000 & 59,071\\
		\bottomrule
	\end{tabular}
\Description[]{Dataset statistics.} 
\end{table}

\subsection{Full Expressiveness}
A link prediction model is fully expressive if for any ground truth over all entities and relations, there exist embeddings that accurately separate the true n-ary relational facts from the false ones, i.e., the link prediction model can recover any given KB tensors by the assignment of entity and relation embeddings \cite{SimplE,balavzevic2019tucker,ComplEx,wang2017knowledge}. 

The full expressiveness guarantees the completeness of link prediction and KB completion. Especially, if a link prediction model is not fully expressive, it means that the model can only represent a part of KBs with prior constraints, which leads to unwarranted inferences \cite{gutierrez2018knowledge}. For instance, DistMult is not fully expressive, and forces relations to be symmetric, i.e., it can represent KBs with only symmetric relations \cite{SimplE}, {\color{black}while KBs with asymmetric and inverse relations cannot be completely represented.} Thus, the upper bound of learning capacity from a not fully expressive model is low. In contrast, fully expressive models enable KB representation with various types of relations, fully representing the knowledge.

Formally, we have the following theorem to establish the full expressiveness of GETD. 
\begin{theorem} \label{theorem:fully expressiveness}
	For any ground truth over entities $\mathcal{E}$ and relations $\mathcal{R}$, there exists a GETD model that represents that ground truth (See proof in Appendix~\ref{sec:appendix-a}).
\end{theorem}

\section{Experiments and Results}\label{sec:experiments}
\subsection{Experimental Setup}
\subsubsection{Datasets}
We evaluate our model with two real datasets across 3-ary and 4-ary relational KBs, one synthetic dataset as well as two benchmark datasets on binary relations, which are introduced as follows.

\textbf{WikiPeople} \cite{guan2019link}: This is a public n-ary relational dataset extracted from Wikidata concerning entities of type $human$. WikiPeople is quite practical, where data incompleteness, insert and update are universal. Due to the sparsity of higher-arity ($\geq 5$) facts in WikiPeople, we filter out all 3-ary and 4-ary relational facts therein, named as WikiPeople-3 and WikiPeople-4, respectively.

\textbf{JF17K} \cite{zhang2018scalable}: This is a public n-ary relational dataset developed from Freebase, whose facts are in good quality. Similar to WikiPeople, the higher-arity facts in JF17K are also sparse, thus we filter out all 3-ary and 4-ary relational facts therein, named as JF17K-3 and JF17K-4, respectively.

\textbf{Synthetic10}: To assess the relationship between the number of parameters and overfitting, we construct the toy dataset across 3-ary and 4-ary relational facts, named as Synthetic10-3 and Synthetic10-4, whose KB tensors are randomly generated by CP decomposition with tensor rank equal to one \cite{kolda2009tensor}. There are only 10 entities and 2 relations in Synthetic10.

\begin{table*}[!htbp]
	\caption{Link prediction results on WikiPeople dataset.}
	\label{tab:results-Wiki}
	\vspace{-10px}
	\begin{tabular}{c|cccc|cccc}
		\toprule
		\multirow{2}{*}{Model} & \multicolumn{4}{c|}{WikiPeople-3} & \multicolumn{4}{c}{WikiePeople-4} \\ \cline{2-9} 
		& MRR & Hits@10 & Hits@3 & Hits@1 & MRR & Hits@10 & Hits@3 & Hits@1 \\
		\hline
		RAE  & 0.239 & 0.379 & 0.252 & 0.168 & 0.150 & 0.273 & 0.149 & 0.080 \\
		NaLP & 0.301 & 0.445 & 0.327 & 0.226 & 0.342 & 0.540 & 0.400 & 0.237 \\
		n-CP & 0.330 & 0.496 & 0.356 & 0.250 & 0.265 & 0.445 & 0.315 & 0.169 \\
		n-TuckER & 0.365 & 0.548 & 0.400 & 0.274 & 0.362 & 0.570 & 0.432 & 0.246 \\
		\hline \hline
		GETD &\textbf{0.373} & \textbf{0.558} & \textbf{0.401} & \textbf{0.284} & \textbf{0.386} & \textbf{0.596} & \textbf{0.462} & \textbf{0.265}\\
		\bottomrule
	\end{tabular}
\Description[]{Link prediction results on WikiPeople dataset.} 
\end{table*}

\begin{table*}[!htbp]
	\caption{Link prediction results on JF17K dataset.}
	\label{tab:results-JF17K}
	\vspace{-10px}
	\begin{tabular}{c|cccc|cccc}
		\toprule
		\multirow{2}{*}{Model} & \multicolumn{4}{c|}{JF17K-3} & \multicolumn{4}{c}{JF17K-4} \\ \cline{2-9} 
		& MRR & Hits@10 & Hits@3 & Hits@1 & MRR & Hits@10 & Hits@3 & Hits@1 \\
		\hline
		RAE  & 0.505 & 0.644 & 0.532 & 0.430 & 0.707 & 0.835 & 0.751 & 0.636 \\
		NaLP & 0.515 & 0.679 & 0.552 & 0.431 & 0.719 & 0.805 & 0.742 & 0.673 \\
		n-CP & 0.700 & 0.827 & 0.736 & 0.635 & 0.787 & 0.890 & 0.821 & 0.733 \\
		n-TuckER & 0.727 & 0.852 & 0.761 & 0.664 & 0.804 & 0.902 & 0.841 & 0.748 \\
		\hline \hline
		GETD &\textbf{0.732} & \textbf{0.856} & \textbf{0.764} & \textbf{0.669} & \textbf{0.810} & \textbf{0.913} & \textbf{0.844} & \textbf{0.755}\\
		\bottomrule
	\end{tabular}
\Description[]{Link prediction results on JF17K dataset.}
\end{table*}

\textbf{WN18} \cite{bordes2014semantic}: This binary relational dataset is a subset of WordNet, a database with lexical relations between words.

\textbf{FB15k} \cite{bordes2013irreflexive}: This binary relational dataset is a subset of Freebase, a database of real world facts including films, sports, etc.

Besides, facts in first three datasets are randomly split into train/valid/test sets by a proportion of 8:1:1. The train/valid/test sets of WN18 and FB15k provided in \cite{transE} are used for evaluation. The datasets statistics are summarized in Table~\ref{tab:dataset}.

\subsubsection{Metrics}
We evaluate the link prediction performance with two standard metrics: mean reciprocal rank (MRR) and Hits@$k$, $k\in\{1,3,10\}$ \cite{SimplE,ComplEx,balavzevic2019tucker,transE,DistMult,guan2019link}. For each testing n-ary relational fact, one of its entities is removed and replaced by all entities in $\mathcal{E}$, leading to $\vert\mathcal{E}\vert$ tuples, which are scored by the link prediction model. The entities in all entity domains are tested. The ranking of the testing fact is obtained by sorting evaluation scores in descending order. MRR is the mean of the inverse of rankings over all testing facts, while Hits@$k$ measures the proportion of top $k$ rankings. Both metrics are in filtered setting \cite{bordes2014semantic}: the ranking of the testing fact is calculated among facts not appeared in train/valid/test sets. The aim is to achieve high MRR and Hits@$k$.

\subsubsection{Baselines}
We compare GETD with the following n-ary relational link prediction baselines:
\begin{itemize}
	\item \textbf{RAE} \cite{zhang2018scalable} is a translational distance model, extending TransH \cite{transH} to m-TransH with relatedness combined\footnote{\url{https://github.com/lijp12/SIR}}.
	\item \textbf{NaLP} \cite{guan2019link} is a neural network model, which achieves the state-of-the-art n-ary relational link prediction performance\footnote{\url{https://github.com/gsp2014/NaLP}}.
	\item \textbf{n-CP} is an extension of CP decomposition \cite{hitchcock1927expression}, firstly applied in n-ary relational link prediction in this paper.
	\item \textbf{n-TuckER} is an extension of TuckER \cite{balavzevic2019tucker} with Tucker decomposition utilized, also firstly applied in n-ary relational link prediction in this paper.
\end{itemize}
Besides, we compare GETD with state-of-the-art models in binary relational KBs, including TransE \cite{transE}, DistMult \cite{DistMult}, ConvE \cite{convE}, ComplEx \cite{ComplEx}, SimplE \cite{SimplE}, and TuckER \cite{balavzevic2019tucker}.

\subsubsection{Implementation}
The implementation of GETD is available at Github\footnote{\url{https://github.com/liuyuaa/GETD}}.
For experimental fairness, we fix entity and relation embedding sizes of GETD, n-CP and n-TuckER. For 3-ary relational datasets WikiPeople-3 and JF17K-3, we set entity and relation embedding sizes to $d_e=d_r=50$,  reshaped tensor order to $k=4$, TR-ranks and TR latent tensor dimensions to $r_i=n_i=50$, while due to the quite smaller numbers of entities and relations, the settings in 4-ary relational datasets are $d_e=d_r=25$, $k=5$, $r_i=n_i=25$. Besides, batch normalization \cite{ioffe2015batch} and dropout \cite{srivastava2014dropout} are used to control overfitting. All hyperparameters except embedding sizes are tuned with \emph{Optuna} \cite{akiba2019optuna}, a Bayesian hyperparameter optimization framework, and the search space of learning rate is $[0.0001,0.1]$ with learning rate decay chosen from $\{0.9,0.995,1\}$, and dropout ranges from $0.0$ to $0.5$. Each model is evaluated with 50 groups of hyperparameter settings. Above three models are trained with Adam \cite{kingma2014adam} using early stopping based on validation set MRR with no improvement for 10 epochs. As for RAE and NaLP, we use the optimal settings reported in \cite{zhang2018scalable} and \cite{guan2019link}, respectively.

\subsection{N-ary Relational Link Prediction}
Table~\ref{tab:results-Wiki} and~\ref{tab:results-JF17K} present the link prediction results on two datasets across 3-ary and 4-ary relational KBs. From the results, we can observe that our proposed GETD achieves the best performance on all metrics across all datasets. Especially, tensor decomposition models of GETD, n-CP and n-TuckER always outperform the translational distance model RAE and the neural network model NaLP. For example, on JF17K, compared with existing state-of-the-art model NaLP, GETD improves MRR by 0.22 and Hits@1 by 55\% for 3-ary relational facts, while the improvement for 4-ary relational facts is 0.09 and 12\%. For WikiPeople, GETD improves MRR by 0.07 and Hits@1 by 25\% on WikiPeople-3, and improves MRR by 0.04 and Hits@1 by 12\% on WikiPeople-4. These considerable improvements further confirm the strong expressive power of the proposed tensor decomposition models. Moreover, the great performance on WikiPeople indicates that GETD is robust and able to handle practical KB issues like data incompleteness, insert and update. Note that the relatively less improvement on 4-ary relational facts may partly owe to the sparsity of higher-arity facts in datasets.

As for the three tensor decomposition models, n-CP is relatively weak due to the difference of embeddings in different entity domains \cite{ComplEx}, while GETD and n-TuckER capture the interaction between entities and relations with TR latent tensors or core tensors. On the other hand, GETD also outperforms n-TuckER owing to the simplicity with much fewer parameters, while the parameter-cost core tensor in n-TuckER increases the complexity of optimization and further overfits. Without the early stopping trick, the performance of n-TuckER seriously degrades and quickly overfits, which is shown in the following. Besides, we run GETD on the largest dataset WikiPeople-3 with a Titan-XP GPU. An epoch training takes about 28s and total training takes 1h, while inference takes only 5s. Overall, the results show the efficiency and robustness of GETD for link prediction in n-ary relational KBs.

\subsection{Overfitting Phenomenon}
Models like n-TuckER with a large amount of parameters easily overfit to the training data, impairing the testing performance. To verify this, we cast the early stopping trick in three tensor decomposition models, and test if there exists the overfitting phenomenon using WikiPeople-4. Accordingly, the training curves in terms of MRR and loss are plotted in Figure~\ref{fig:overfit-mrr} and \ref{fig:overfit-loss}, respectively. 

\begin{figure}[!htbp]
	\subfigure[MRR v.s. epoch.]{ \label{fig:overfit-mrr} 
		\includegraphics[width=0.23\textwidth]{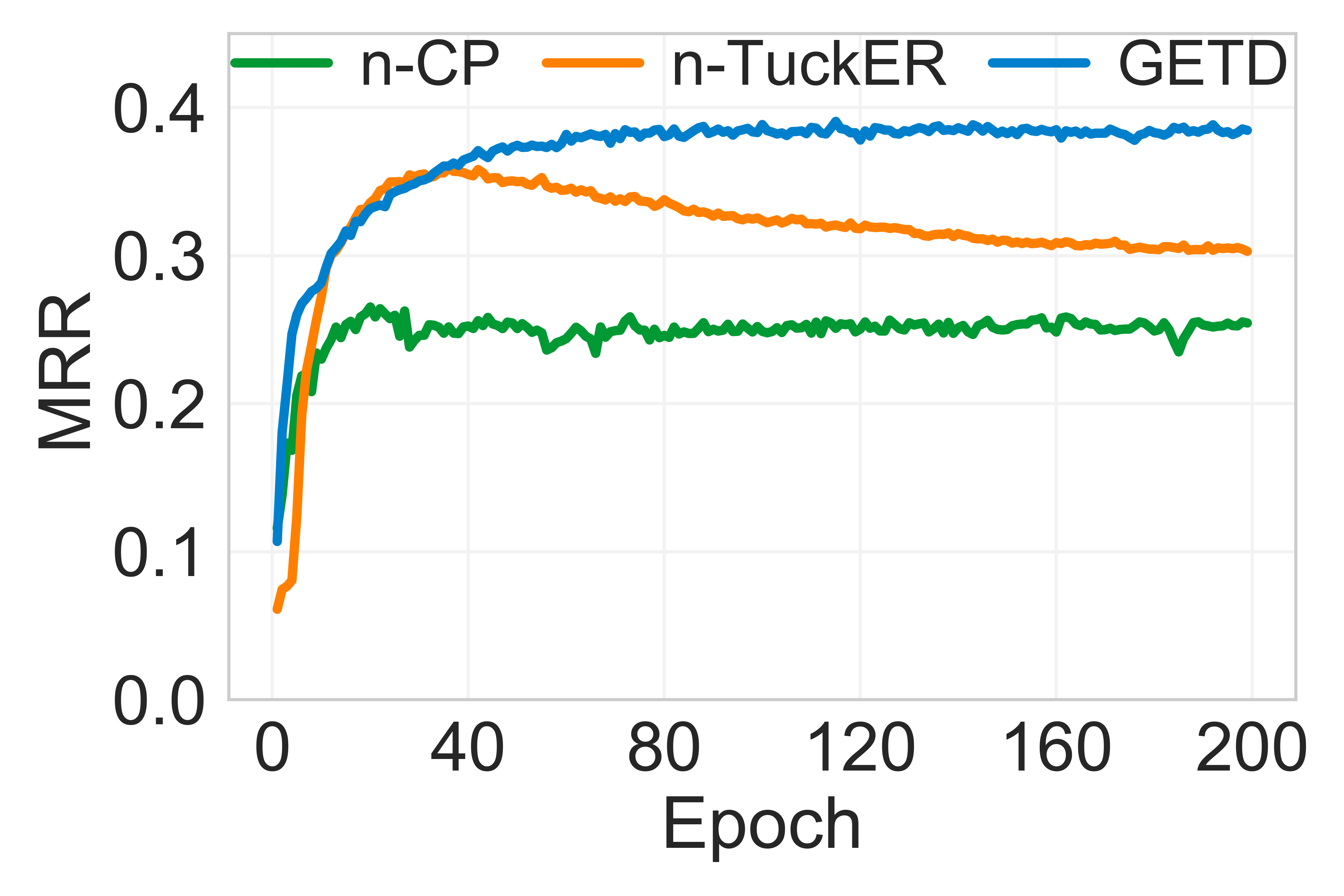}}
	\subfigure[Loss v.s. epoch.]{ \label{fig:overfit-loss} 
		\includegraphics[width=0.23\textwidth]{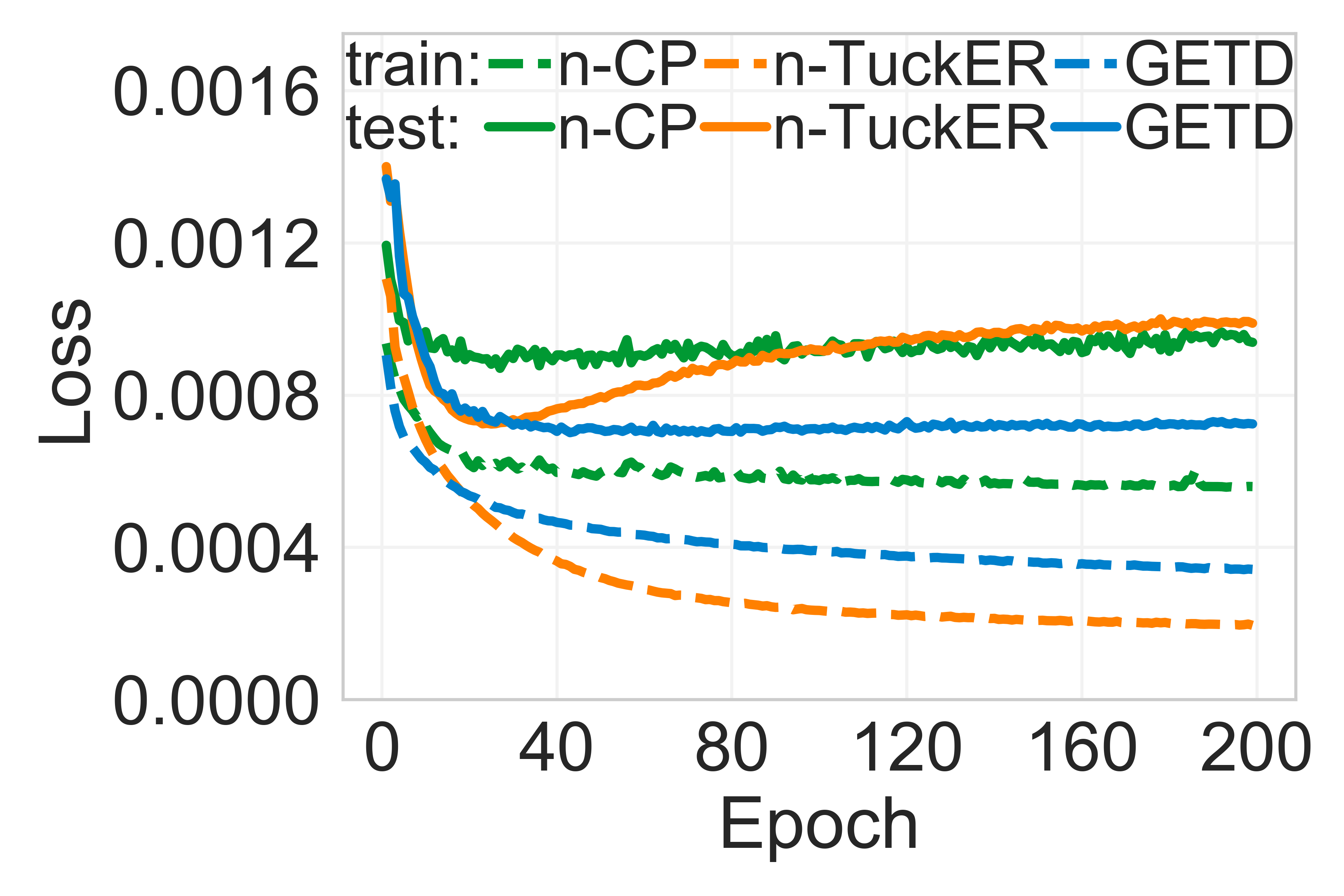}}
	\vspace{-10px}
	{\caption{Overfitting in n-TuckER observered from MRR (left) and loss (right). Evaluated on WikiPeople-4.}
		\label{fig:overfit-WikiPeople}}
	\Description[]{Overfitting phenomenon.}
\end{figure}

From the results, we can clearly observe the overfitting phenomenon in training process of n-TuckER. In Figure~\ref{fig:overfit-mrr}, as training going on, the MRR of n-TuckER increases first and then quickly decreases, while the MRR of the other two models increases to convergence. Moreover, GETD outperforms n-CP due to its strong expressive power. As for loss curves, the train losses of all three models keep decreasing, while the test loss of n-TuckER increases after 20 epochs of training, compared with the convergence in GETD and n-CP test loss curves. It mainly caused by the model complexity that, the numbers of parameters in GETD and n-CP are 0.4 million and 0.9 million, while 10 million in n-TuckER.  

\begin{figure}[!htbp]
	\subfigure[Synthetic10-3.]{ \label{fig:toy-3ar} 
		\includegraphics[width=0.23\textwidth]{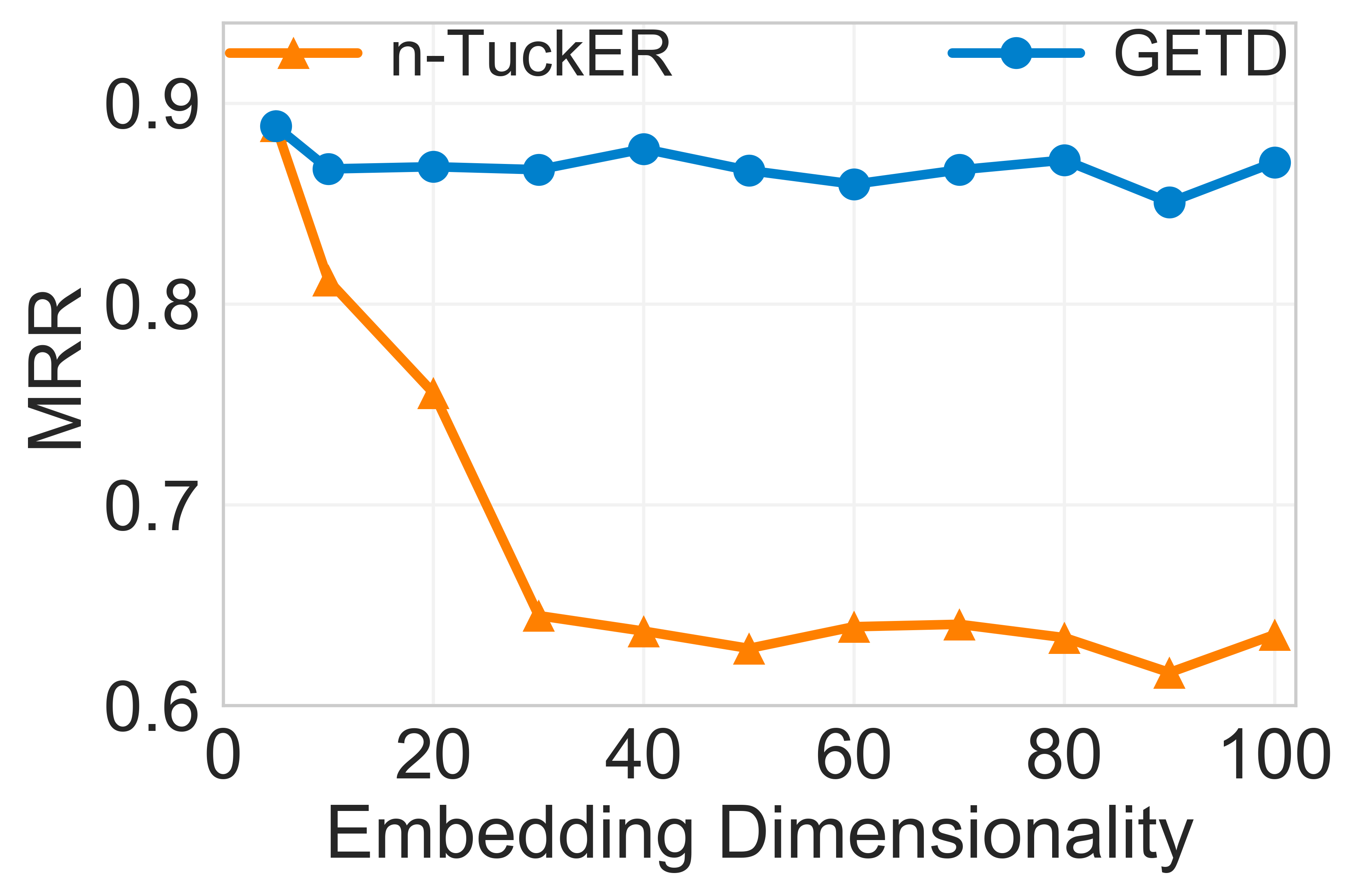}}
	\subfigure[Synthetic10-4.]{ \label{fig:toy-4ar} 
		\includegraphics[width=0.23\textwidth]{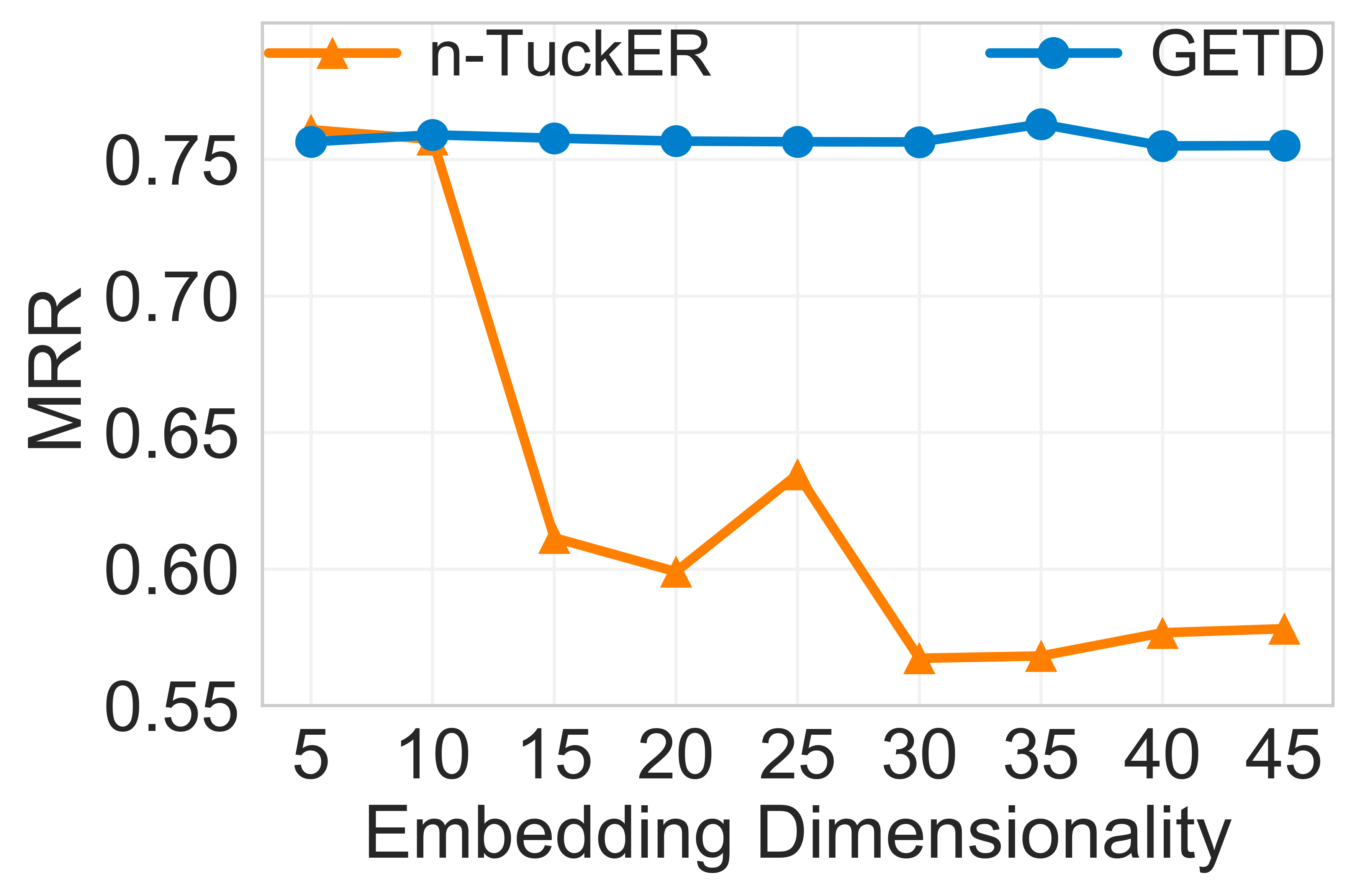}}
	\vspace{-10px}
	{\caption{MRR for n-TuckER and GETD for different embeddings sizes. Evaluated on Synthetic10-3 (left) and Synthetic10-4 (right).}
		\label{fig:overfit-toy}}
	\Description[]{MRR for n-TuckER and GETD for different embeddings sizes. Evaluated on Synthetic10-3 (left) and Synthetic10-4 (right).}
\end{figure}

To further reveal the relationship between the overfitting and the number of parameters, we evaluate the MRR for different embedding sizes on Synthetic10 in Figure~\ref{fig:overfit-toy}. The early stopping is cast, and the MRR after 200 epochs of training are reported. The results of n-TuckER in both 3-ary and 4-ary relational datasets show that, increasing of embedding sizes results in a quality fall in the case of MRR, which means overfitting of n-TuckER. This phenomenon is mainly caused by low-rank property of KB tensors. Taking $d_e=d_r=50$ in Synthetic10-3 as an example, the core tensor in n-TuckER costs $50\times 50 \times 50 \times 50=6.25$ million parameters, while the TR latent tensors in GETD cost only $4\cdot 1\times 50\times 1=200$ parameters with TR-ranks equal to one. Therefore, using n-TuckER with large embedding sizes to approximate the low-rank KB tensors is intractable and prone to overfitting. However, for general KBs, embedding sizes should be large enough for strong expressive power, which is a contradiction. In comparison, GETD is capable of coping with overfitting and expressiveness together based on embedding sizes as well as TR-ranks, which is much more flexible. The flexibility as well as expressiveness thus support the great performance of GETD in Table~\ref{tab:results-Wiki} and ~\ref{tab:results-JF17K}. 

\subsection{Influence of Parameters}
\begin{table*}[tbp]
	\caption{Link prediction results on WN18 and FB15k. Results of TransE and DistMult are copied from \cite{ComplEx}. Other results are copied from the corresponding original papers \cite{convE,ComplEx,SimplE,balavzevic2019tucker}}.
	\label{tab:results-KG}
	\vspace{-10px}
	\begin{tabular}{c|cccc|cccc}
		\toprule
		\multirow{2}{*}{Model} & \multicolumn{4}{c|}{WN18} & \multicolumn{4}{c}{FB15k} \\ \cline{2-9} 
		& MRR & Hits@10 & Hits@3 & Hits@1 & MRR & Hits@10 & Hits@3 & Hits@1 \\
		\hline
		TransE \cite{transE}  & 0.454 & 0.934& 0.823& 0.089& 0.380& 0.641& 0.472& 0.231\\
		DistMult \cite{DistMult}  & 0.822& 0.936& 0.914& 0.728& 0.654 & 0.824& 0.733& 0.546 \\
		ConvE \cite{convE} & 0.943& 0.956& 0.946& 0.935& 0.657 & 0.831& 0.723& 0.558 \\
		ComplEx \cite{ComplEx}   & 0.941& 0.947& 0.945& 0.936& 0.692& 0.840& 0.759& 0.599\\
		SimplE \cite{SimplE} & 0.942& 0.947& 0.944& 0.939& 0.727& 0.838& 0.773& 0.660\\
		TuckER \cite{balavzevic2019tucker} & \textbf{0.953}& \textbf{0.958} & \textbf{0.955}& \textbf{0.949}& 0.795& \textbf{0.892}& 0.833& 0.741\\
		\hline \hline
		GETD & 0.948 & 0.954 & 0.950 & 0.944 & \textbf{0.824} & 0.888 & \textbf{0.847} & \textbf{0.787}\\
		\bottomrule
	\end{tabular}
\Description[]{Link prediction results on WN18 and FB15k datasets.}
\end{table*}

Since the embedding sizes are important factors to link prediction models with expressiveness \cite{balavzevic2019tucker,ComplEx,convE}, while TR-ranks, as well as the reshaped tensor order are unique hyper-parameters of GETD, determining the model complexity and performance, now we investigate the impacts of these parameters. 

\subsubsection{Influence of Embedding Sizes $d_e,d_r$}
The MRR and the number of parameters of three tensor decomposition models under different embedding sizes are evaluated on WikiPeople-4 with TR-ranks equal to embedding sizes. The results are plotted in Figure~\ref{fig:edim-WikiPeople}. 

\begin{figure}[!htbp]
	\subfigure[MRR v.s. $d_e,d_r$.]{ \label{fig:edim-mrr} 
		\includegraphics[width=0.23\textwidth]{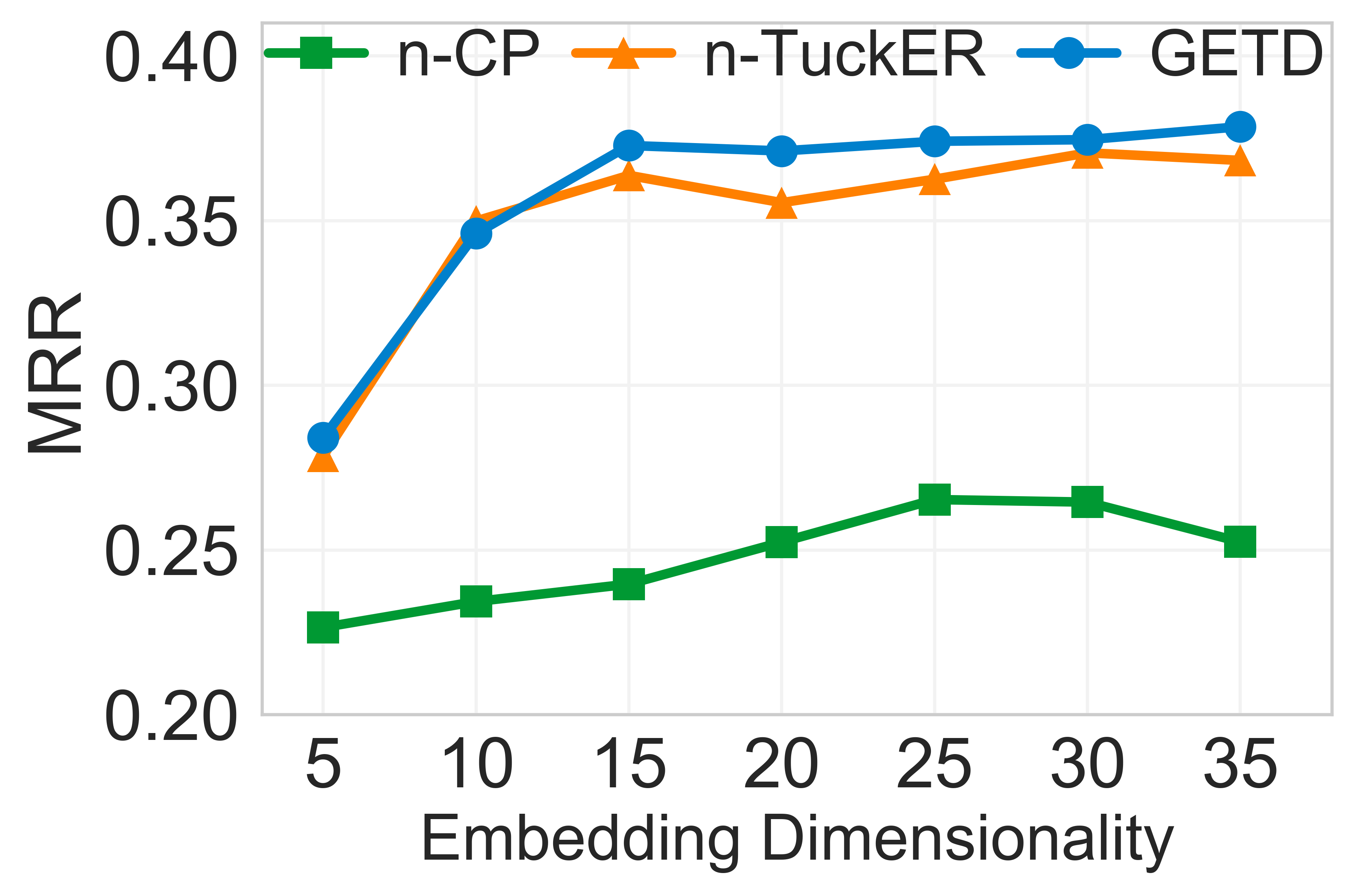}}
	\subfigure[\#Paramters v.s. $d_e,d_r$.]{ \label{fig:edim-param} 
		\includegraphics[width=0.23\textwidth]{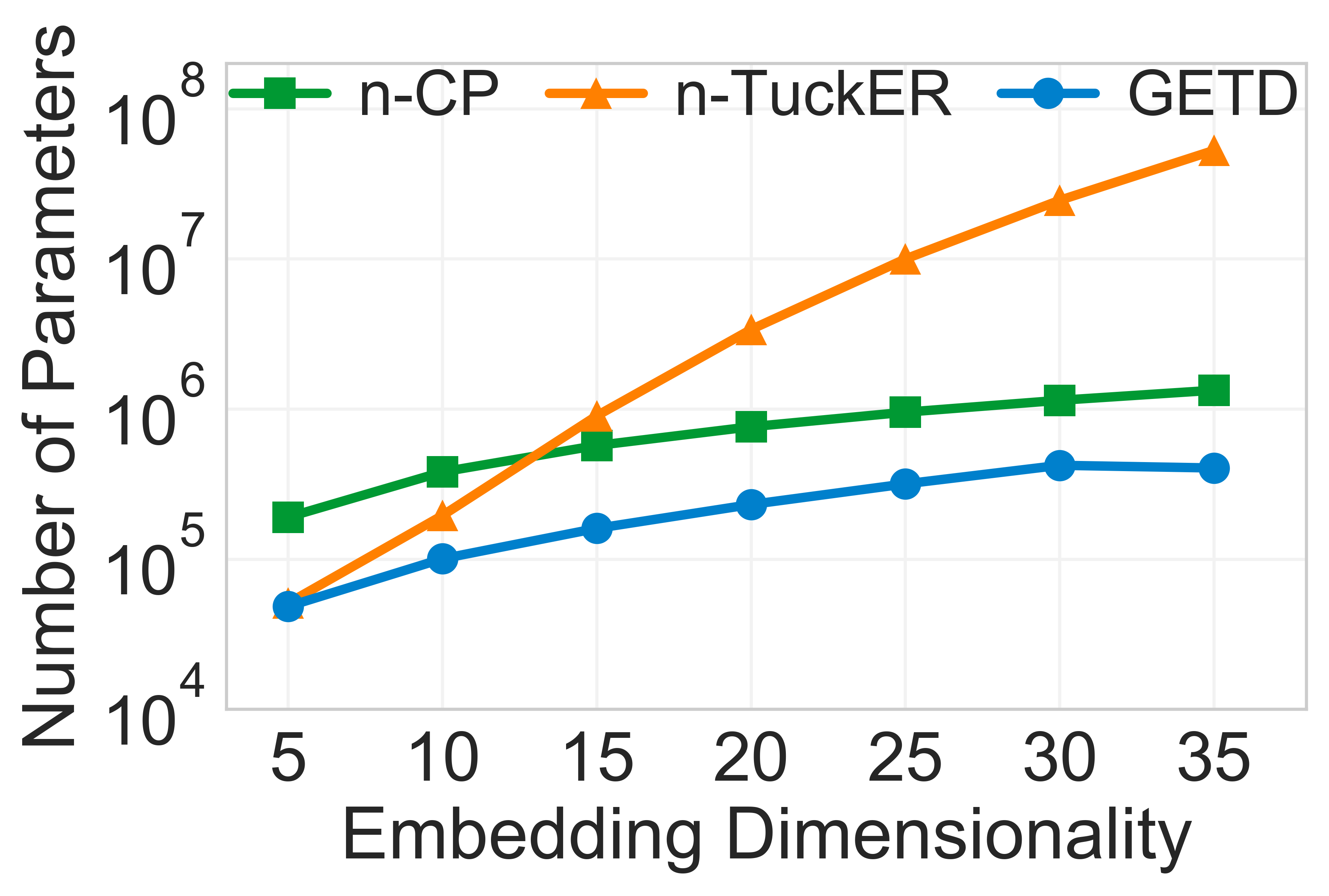}}
	\vspace{-10px}
	{\caption{MRR (left) and number of parameters (right) under different embedding sizes. Evaluated on WikiPeople-4.}
		\label{fig:edim-WikiPeople}}
	\Description[]{MRR (left) and number of parameters (right) under different embedding sizes. Evaluated on WikiPeople-4.}
\end{figure}

According to Figure~\ref{fig:edim-mrr}, GETD always outperforms n-TuckER and n-CP. The MRR of GETD increases globally with the increase of embedding sizes, and gradually becomes smooth. While the MRR of n-TuckER is not stable even with large embedding sizes when early stopping applied. For example, at embedding size 35, GETD increases MRR by $49\%$ for n-CP, and $2.8\%$ for n-TuckER. Moreover, the MRR of GETD reaches $0.372$ at embedding size 15, which is better than the performance of n-TuckER at embedding size over $30$. On the other hand, GETD uses the least parameters in three models, which is shown in Figure~\ref{fig:edim-param}. For embedding size 30, n-TuckER costs 24 million parameters with core tensor using $30^5=24.3$ million, n-CP costs 1.14 million parameters, while GETD only costs 0.42 million parameters with TR latent tensors using $5\cdot 30^3=0.13$ million, $1.7\%$ of n-TuckER parameters. {\color{black}The results are accord with complexity analysis in Section~\ref{sec:theoretical_analysis}, and further indicate that GETD with relatively small embedding sizes is able to obtain good performance, which can be applied for large-scale KBs.}

\subsubsection{Influence of TR-ranks $\bm{r}$}
Since the TR-ranks can largely determine the number of TR latent tensor parameters, and make GETD model more flexible, we reveal the relationship between link prediction performance and TR-ranks on WikiPeople-4 and JF17K-3, as shown in Figure~\ref{fig:mrr-TRrank}. From the results, we can observe that the link prediction performance is affected only when TR-ranks are very small (less than 5), indicating that GETD is not sensitive to TR-ranks. When TR-ranks vary from 20 to 60 on JF17K-3, the MRR is rather stable, and a similar trend can be found on WikiPeople-4. This implies that TR tensors with TR-ranks about 20 are often enough to capture the latent interactions between entities and relations for given datasets. Based on this, the number of parameters for GETD can be further reduced to control model complexity for large-scale KBs.

\begin{figure}[!htbp]
	\subfigure[WikiPeople-4.]{ \label{fig:rank-4ar} 
		\includegraphics[width=0.23\textwidth]{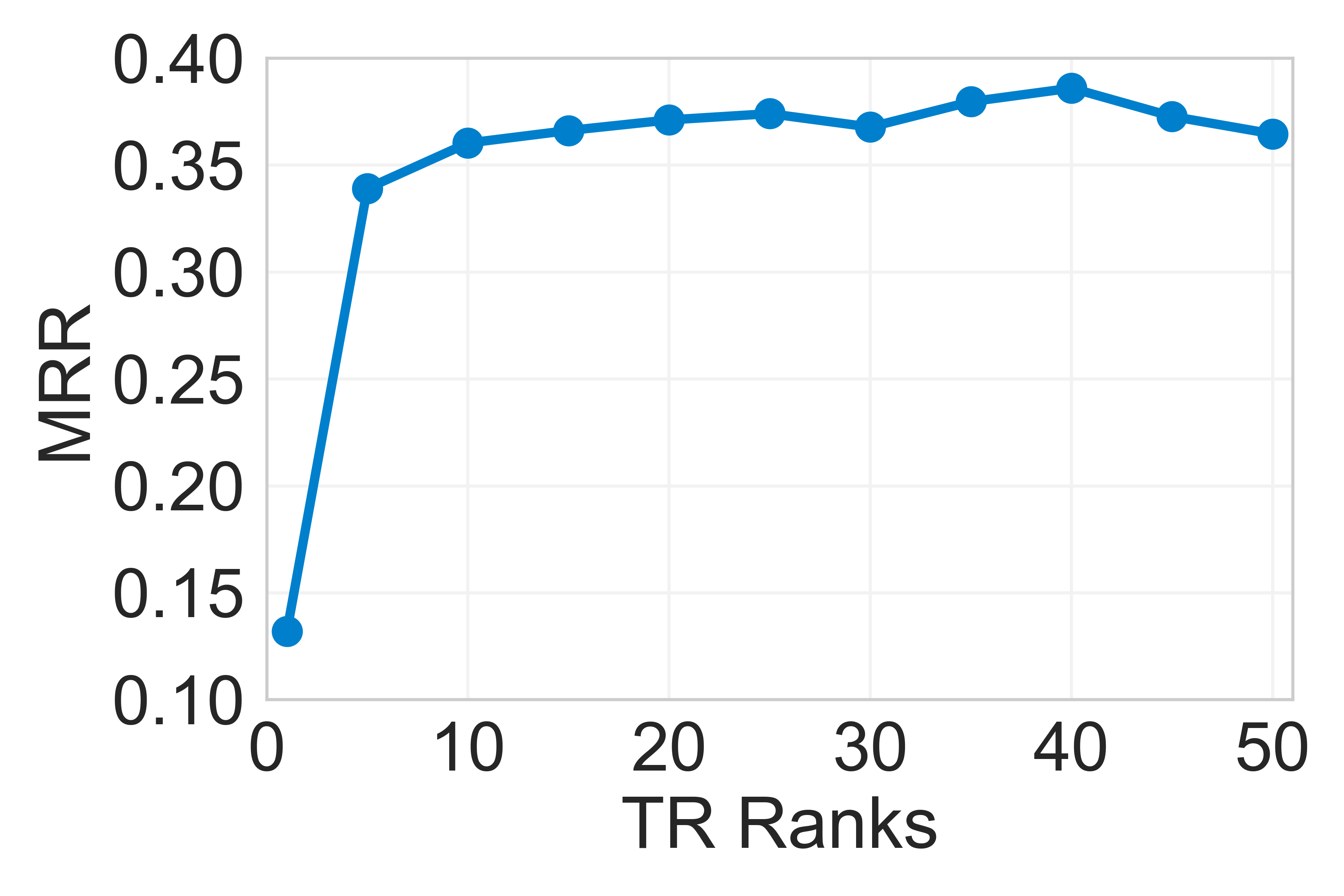}}
	\subfigure[JF17K-3.]{ \label{fig:rank-3ar} 
		\includegraphics[width=0.23\textwidth]{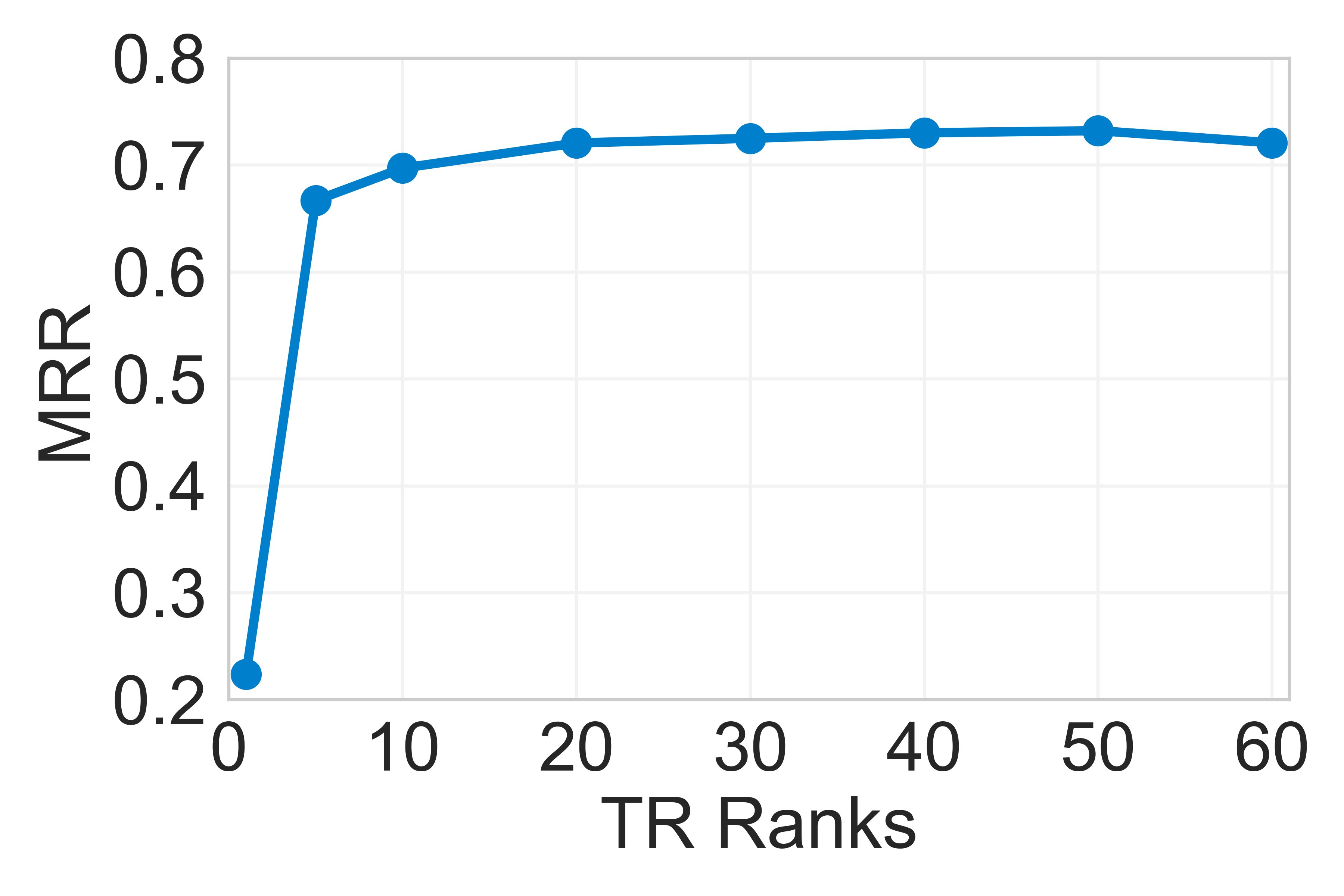}}
	\vspace{-10px}
	{\caption{MRR of GETD on WikiPeople-4 (left) and JF17K-3 (right) under different TR-ranks.}
		\label{fig:mrr-TRrank}}
	\Description[]{MRR of GETD on WikiPeople-4 (left) and JF17K-3 (right) under different TR-ranks.}
\end{figure}
\subsubsection{Influence of Reshaped Tensor Order $k$}
\begin{figure}[htbp]
	\subfigure[WikiPeople-4, $d_e=d_r=25$.]{ \label{fig:order-4ar} 
		\includegraphics[width=0.23\textwidth]{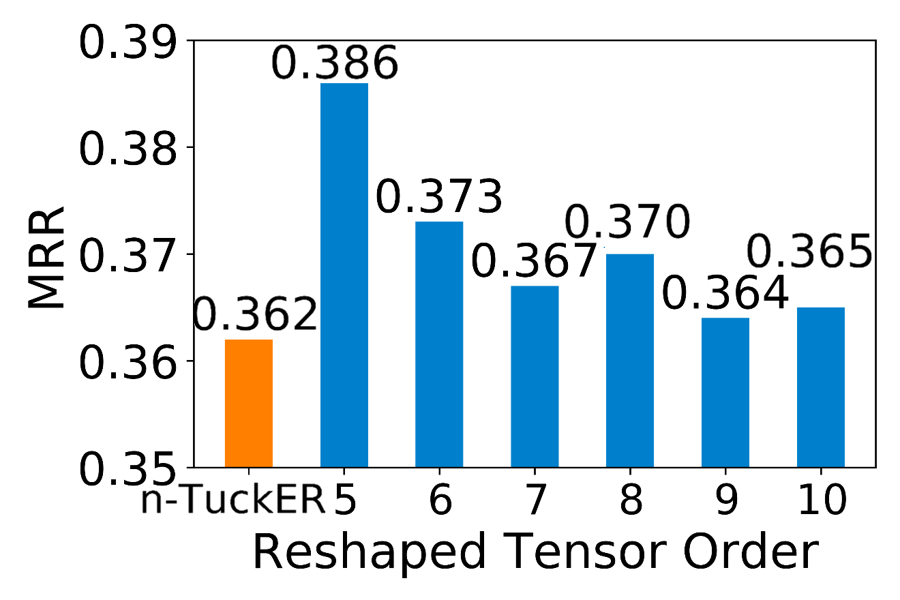}}
	\subfigure[JF17K-3, $d_e=d_r=64$.]{ \label{fig:order-3ar} 
		\includegraphics[width=0.23\textwidth]{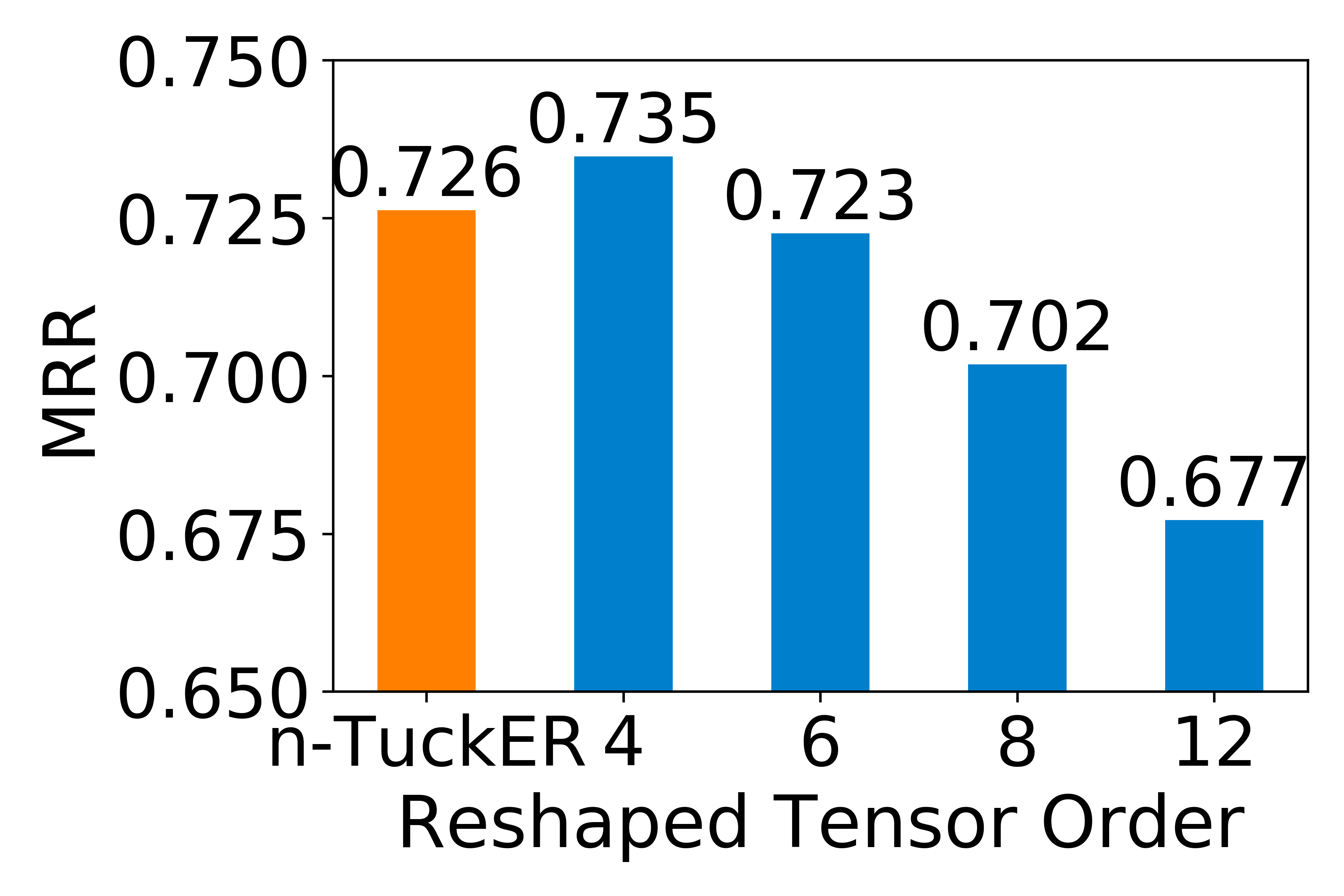}}
	\vspace{-10px}
	{\caption{MRR of GETD on WikiPeople-4 (left) and JF17K-3 (right) under different reshaped tensor orders.}
		\label{fig:mrr-TRorder}}
	\Description[]{MRR of GETD on WikiPeople-4 (left) and JF17K-3 (right) under different reshaped tensor orders.}
\end{figure}
As a key step of connecting Tucker and TR in GETD, the effect of reshaped tensor order is investigated on WikiPeople-4 and JF17K-3, exhibited in Figure~\ref{fig:mrr-TRorder}.  

For WikiPeople-4, the embedding size is set to 25 and thus the original core tensor is $\bm{{\mathscr{W}}}\in\mathbb{R}^{25\times 25 \times 25 \times 25 \times 25}$, leading to the 5th-order reshaped tensor $\bm{\hat{\mathscr{W}}}\in\mathbb{R}^{25\times 25 \times 25 \times 25 \times 25}$, the 6th-order reshaped tensor $\bm{\hat{\mathscr{W}}}\in\mathbb{R}^{5\times 25 \times 25 \times 25 \times 25\times 5}$, etc. It can be observed that, GETD with different orders of reshaped tensors always achieves higher MRR compared with n-TuckER, and the best one increases MRR by 0.024, which is decomposed by 5 TR latent tensors. Overall, the expressive power of GETD decreases with the increase of reshaped tensor order. 

As for JF17K-3, the embedding size is set to 64 so that the reshaped tensor is cubic, i.e., each mode is in the same size \cite{kolda2009tensor}. For example, the 6th-order reshaped tensor becomes $\bm{\hat{\mathscr{W}}}\in\mathbb{R}^{16\times 16 \times 16 \times 16 \times 16 \times 16}$, and the size of each mode for 8th-order tensor becomes 4. Similarly, GETD with the least order of reshaped tensor achieves the highest MRR. Moreover, Figure~\ref{fig:order-3ar} clearly shows the negative correlation between MRR and reshaped tensor order, which is in accord with the above results. This phenomenon is mainly because the higher order involves more TR latent tensors, increasing the optimization complexity. On the other hand, since GETD requires that $\prod_{i=1}^kn_i=d^n_ed_r$, the higher order $k$ also means the smaller $n_{\max}$, which reduces the number of parameters. Thus, the reshaped tensor order in GETD should be appropriately determined considering both link prediction performance and model complexity. 

\subsection{Binary Relational Link Prediction}

To investigate the robustness as well as representation capability of our proposed GETD model, we evaluate the performance of GETD model on WN18 and FB15k. The experimental settings are the same as in n-ary relational link prediction. The embedding sizes of GETD are set to 200, which is similar to the setting in TuckER \cite{tucker1966some}. The reshaped tensor order $k$ is 3, TR-ranks $r_i$ and TR latent tensor dimensions $n_i$ are set to 50 and 200, respectively. 

Table~\ref{tab:results-KG} summarizes the results of GETD and the state-of-the-art models on two datasets. According to the results, GETD achieves the second best performance on WN18, with a quite small MRR gap of 0.005 to TuckER. Moreover, TR latent tensors in GETD costs $3\cdot 50\times 200 \times 50=1.5$ million parameters, only $1/8$ of core tensor parameters in TuckER ($200\times 200 \times 200=8$ million). Thus, GETD is able to obtain better performance with larger embedding sizes but similar number of parameters. As for FB15k, GETD outperforms all state-of-the-art models, and increases MRR by 0.03. Also, GETD increases the toughest metric Hits@1 by 4\% on FB15k. These results demonstrate that, GETD is robust and works well in representing KBs with different arity relations.

\section{Conclusion}\label{sec:conclusion}
This work proposed GETD, a fully expressive tensor decomposition model for link prediction in n-ary relational KBs. Based on the expressiveness of Tucker decomposition as well as the flexibility of tensor ring decomposition, GETD is able to capture the latent interactions between entities and relations with a small number of parameters. Experimental results demonstrate that GETD outperforms the state-of-the-art models for n-ary relational link prediction and achieves close and even better performance on standard binary relational KB datasets.

Considering the benefits of parameter reduction with higher-order reshaped tensor, we plan to extend GETD with appropriate optimization techniques so that GETD with higher-order reshaped tensor can also achieve comparable performance. Besides, GETD only uses observed facts for link prediction, while incorporating GETD with background knowledge such as logical rules and entity properties may bring performance enhancement.
Finally,
we also consider 
using automated machine learning techniques to search
the scoring function from the data \cite{yao2018taking,zhang2019autokge}.

\begin{acks}
	This work was supported in part by The National Key Research and Development Program of China under grant 2018YFB1800804, the National Nature Science Foundation of China under U1836219, 61971267, 61972223, 61861136003, Beijing Natural Science Foundation under L182038, Beijing National Research Center for Information Science and Technology under 20031887521, and research fund of Tsinghua University - Tencent Joint Laboratory for Internet Innovation Technology. 
\end{acks}

\appendix
\section{Proofs} \label{sec:appendix-a}
\subsection{Preliminaries}
Now, we first introduce lemmas that will be used later in the proof.

\begin{lemma}\label{lemma:tucker}
	For any ground truth over entities $\mathcal{E}$ and relations $\mathcal{R}$, there exists an n-TuckER model with entity embeddings of dimensionality $d_e=\vert \mathcal{E}\vert$ and relation embeddings of dimensionality $d_r=\vert \mathcal{R}\vert$ that represents that ground truth.
\end{lemma}

\begin{proof}	
	Let $\bm{e}_{i_1},\bm{e}_{i_2},\cdots,\bm{e}_{i_n}$ be the $n_e$-dimensional one-hot binary vector representation of entities $i_1,i_2,\cdots,i_n$, and $\bm{r}_{i_r}$ the $n_r$-dimensional one-hot binary vector representation of a relation $i_r$. We let the $i_r$-th, $i_1$-th, $i_2$-th, $\cdots$ $i_n$-th element respectively of the corresponding vectors $\bm{r}_{i_r},\bm{e}_{i_1},\bm{e}_{i_2},\cdots,\bm{e}_{i_n}$ be 1 and all other elements 0. Further, we set the $w_{i_ri_1i_2\cdots i_n}$ of the core tensor to 1 if the fact $(i_r,i_1,i_2,\cdots,i_n)$ holds and 0 otherwise. Thus the tensor product of these entity embeddings and the relation embedding with the core tensor, accurately represents the original KB tensor.
\end{proof}

\begin{lemma}\label{lemma:CP}
	Given any $k$th-order binary tensor $\bm{\mathscr{W}}\in\{0,1\}^{n_1\times n_2\times\cdots\times n_k}$, there exists a CP model with rank $r_{CP}=\prod^k_{i=1}n_i$, that completely decomposes $\bm{\mathscr{W}}$.
\end{lemma}

\begin{proof}
	Since $\bm{\mathscr{W}}\in\{0,1\}^{n_1\times n_2\times\cdots\times n_k}$, we can use $r_{CP}=\prod^k_{i=1}n_i$ zero/one-hot tensors $\{\bm{\mathscr{W}}^{(r)}\,|\,\bm{\mathscr{W}}^{(r)}\in\{0,1\}^{n_1\times n_2\times\cdots\times n_k}\}^{r_{CP}}_{r=1}$ to represent $\bm{\mathscr{W}}$, s.t., $w^{(1)}_{1\cdots1}=w_{1\cdots 1}$, while other elements in $\bm{\mathscr{W}}^{(1)}$ are zeros, $w^{(r)}_{j^{(r)}_1j^{(r)}_2\cdots j^{(r)}_k}=w_{j^{(r)}_1j^{(r)}_2\cdots j^{(r)}_k}$, while other elements in $\bm{\mathscr{W}}^{(r)}$ are zeros, etc. Finally, $w^{(r_{CP})}_{n_1n_2\cdots n_k}=w_{n_1n_2\cdots n_k}$, while other elements in $\bm{\mathscr{W}}^{(r_{CP})}$ are zeros. Moreover, $\bm{\mathscr{W}}^{(r)}$ can be decomposed by rank-one CP decomposition as, 
	\begin{align}
	& \bm{\mathscr{W}}^{(r)}=\bm{u}^{(1)}_r\circ\bm{u}^{(2)}_r\circ\cdots\circ \bm{u}^{(k)}_r,\\
	\text{s.t.}\ \ &   {w}^{(r)}_{j^{r}_1j^{r}_2\cdots j^{r}_k}={u}^{(1)}_r(j^{r}_1)\cdot {u}^{(2)}_r(j^{r}_2)\cdot\cdots \cdot {u}^{(k)}_r(j^{r}_k)
	\end{align}
	where $\bm{u}^{(i)}_{r}$ is the $n_i$-dimensional zero/one-hot vector, and ${u}^{(i)}_{r}(j)$ is the $j$-th element of the vector. Therefore, by assigning ${u}^{(1)}_r(j^r_1)={u}^{(2)}_r(j^r_2)=\cdots={u}^{(k)}_r(j^r_k)={w}^{(r)}_{j^{r}_1j^{r}_2\cdots j^{r}_k}$, and other elements in $\bm{u}^{(i)}_r$ being zeros, the binary tensor $\bm{\mathscr{W}}$ can be completely decomposed by CP decomposition via the set of vectors $\{\bm{u}^{(i)}_r\,|\, i=1,2,\cdots k, r=1,2,\cdots,r_{CP}, r_{CP}=\prod^k_{i=1}n_i\}$ as,
	\begin{align}
	\bm{\mathscr{W}}=\sum^{r_{CP}}_{r=1} \bm{\mathscr{W}}^{(r)}=\sum^{r_{CP}}_{r=1}\bm{u}^{(1)}_r\circ\bm{u}^{(2)}_r\circ\cdots\circ \bm{u}^{(k)}_r. \label{eq:cp}
	\end{align}		
\vspace{-5px}
\end{proof}

\begin{lemma}\label{lemma:TR-CP}
	Cannonical polyadic (CP) decomposition can be viewed as a special case of tensr ring (TR) decomposition. Given a $k$th-order binary tensor $\bm{\mathscr{W}}\in\{0,1\}^{n_1\times n_2\times\cdots\times n_k}$ with its CP decomposition as \eqref{eq:cp}, it can also be written in TR decomposition form as,
	\begin{align}
	&\bm{\mathscr{W}}=\sum^{r_{CP}}_{r=1}\bm{u}^{(1)}_r\circ\bm{u}^{(2)}_r\circ\cdots\circ \bm{u}^{(k)}_r=\bm{TR}(\bm{\mathscr{Z}}_1,\bm{\mathscr{Z}}_2,\cdots,\bm{\mathscr{Z}}_k),\\
	\text{\rm s.t.} \ \ & \bm{Z}_i(j_i)=diag(\bm{u}^{(i)}(j_i)), \ \ \forall i=1,2,\cdots,m
	\end{align}
	where $\bm{\mathscr{Z}}_i\in \{0,1\}^{r_{CP}\times n_i\times r_{CP}}$, $\bm{u}^{(i)}(j_i)=[{u}_1^{(i)}(j_i),{u}_2^{(i)}(j_i),\cdots,{u}_{r_{CP}}^{(i)}(j_i)]$. See proof in \cite{zhao2016tensor}.
\end{lemma}

\subsection{Theorem~\ref{theorem:fully expressiveness}}
\begin{proof}
	According to Lemma~\ref{lemma:tucker}, n-TuckER is fully expressive by setting the embeddings as well as the core tensor, in which the core tensor is set to an $(n+1)$-th order binary tensor $\bm{\mathscr{W}}\in\{0,1\}^{n_r\times n_e\times n_e\times\cdots\times n_e}$. 
	
	In GETD, $\bm{\mathscr{W}}$ is reshaped into a $k$th-order reshaped tensor $\hat{\bm{\mathscr{W}}}\in\{0,1\}^{n_1\times \cdots \times n_k}$, which is further decomposed by TR decomposition. Keeping the embedding settings as the ones in Lemma~\ref{lemma:tucker},  we only need to prove that TR decomposition is able to recover any given tensor $\hat{\bm{\mathscr{W}}}$. On the other hand, with Lemma~\ref{lemma:CP}, $\hat{\bm{\mathscr{W}}}$ is able to be completely recovered via CP decomposition. Moreover, the CP decomposition can be written as a special case of TR decomposition by Lemma~\ref{lemma:TR-CP}, which derives TR latent tensors $\{\bm{\mathscr{Z}}_i\,|\,\bm{\mathscr{Z}}_i\in \{0,1\}^{r_{CP}\times n_i\times r_{CP}}\}^k_{i=1}$. Overall, following the settings of embeddings in Lemma~\ref{lemma:tucker} and TR latent tensors in Lemma~\ref{lemma:TR-CP}, GETD is proved to be fully expressive with entity embeddings of dimensionality $n_e=\vert\mathcal{E}\vert$ and relation embeddings of dimensionality $d_r=\vert\mathcal{R}\vert$.
\end{proof}

\bibliographystyle{ACM-Reference-Format}
\bibliography{sample-base}

\end{document}